\RequirePackage{amsthm}
\documentclass[pdflatex,sn-mathphys,Numbered]{sn-jnl}

\usepackage{graphicx}%

\graphicspath{{images/}}
\usepackage{multirow}%
\usepackage{amsmath,amssymb,amsfonts}%
\usepackage{mathrsfs}%
\usepackage[title]{appendix}%
\usepackage{xcolor}%
\usepackage{textcomp}%
\usepackage{manyfoot}%
\usepackage{booktabs}%
\usepackage{algorithm}%
\usepackage{algorithmicx}%
\usepackage{algpseudocode}%
\usepackage{listings}%
\usepackage{hyperref}
\hypersetup{
	colorlinks = true,  
	linkcolor = blue,  
	citecolor = blue,  
	urlcolor = black    
}

\theoremstyle{thmstyleone}%
\newtheorem{theorem}{Theorem}
\newtheorem{proposition}{Proposition}%

\theoremstyle{thmstyletwo}%

\theoremstyle{thmstylethree}%
\newtheorem{lemma}{Lemma}
\DeclareMathOperator*{\argmax}{arg\,max}

\begin{document}

\title[Article Title]{ICM Ensemble with Novel Betting Functions for Concept Drift}


\author*[1]{\fnm{Charalambos} \sur{Eliades}}\email{pamboseliades@hotmail.com}

\author[1]{\fnm{Harris } \sur{Papadopoulos}}\email{h.papadopoulos@frederick.ac.cy}
\equalcont{These authors contributed equally to this work.}

\affil*[1]{\orgdiv{Computational Intelligence (COIN) Research Lab}, \orgname{Frederick University}, \orgaddress{\street{7, Y. Frederickou Str., Pallouriotissa}, \city{Nicosia}, \postcode{1036},\country{Cyprus}}}


\abstract{This study builds upon our previous work by introducing a refined Inductive Conformal Martingale (ICM) approach for addressing Concept Drift (CD). Specifically, we enhance our previously proposed CAUTIOUS betting function to incorporate multiple density estimators for improving detection ability. We also combine this betting function with two base estimators that have not been previously utilized within the ICM framework: the Interpolated Histogram and Nearest Neighbor Density Estimators. We assess these extensions using both a single ICM and an ensemble of ICMs. For the latter, we conduct a comprehensive experimental investigation into the influence of the ensemble size on prediction accuracy and the number of available predictions. Our experimental results on four benchmark datasets demonstrate that the proposed approach surpasses our previous methodology in terms of performance while matching or in many cases exceeding that of three contemporary state-of-the-art techniques.}

\noindent\textit{This manuscript is currently under review at the Machine Learning journal.}



\maketitle

\section{Introduction}\label{sect:introduction}

CD corresponds to a change in the underlying data generating mechanism, resulting in  loss of classification performance. Using ICM on the calculated p-values of each classifier in the ensemble, we detect violations of the Exchangeability Assumption (EA) for a pre-specified significance level and retrain the corresponding classifier of the ensemble to regain performance.

Formally a CD   can be defined as follows:  Given a data stream     $S=\{(x_0,y_0), (x_1,y_1),\dots,\}$, where $x_i$ is a feature vector and $y_i$ a label; if  $S$ can be divided into two sets $S_{0:t} = \{(x_0, y_0), \dots, (x_t,y_t)\}$ and $S_{t+1:} = \{(x_{t+1}, y_{t+1}), \dots\}$, such that $S_{0:t}$ and $S_{t+1: \dots}$ are generated by two different distributions, then a CD occurred at timestamp $t+1$. This can be extended to any number of CDs with any number of different distributions. 

CD can be produced from three sources. Specifically, the joint probability density function of a pair $(x,y)$ at time $t$ is denoted by $f_{X,Y,t} = f_{Y|X,t} \cdot f_{X,t}$, and similarly, $f_{X,Y,t+1} = f_{Y|X,t+1} \cdot f_{X,t+1}$ represents these functions at time $(t+1)$.  A change in the joint distribution of (X,Y) can be the result of one of the following: (a) $f_{X|Y,t}=f_{Y|X,t+1}$ and   $f_{Y,t} \neq  f_{Y,t+1}$, in this case we have a change in the label's distribution while the decision boundaries remain unchanged, this is also known as  virtual drift \citep{survey2} (b) $f_{Y|X,t} \neq f_{Y|X,t+1}$ and   $f_{Y,t}=f_{Y,t+1}$ here the decision boundaries change and lead to decrease in accuracy, also referred to as actual drift \citep{survey2}  and (c) $f_{Y|X,t} \neq f_{Y|X,t+1}$ and   $f_{Y,t}\neq f_{Y,t+1}$ which is a combination of the two previously mentioned sources. 

CD types can be classified into four categories: (a) sudden drift, where the data generating mechanism changes instantly (b) gradual drift, where the data distribution is replaced with a new one over time (i.e. each example is generated by a mixture of distributions but over time the impact of the initial distribution disappears),   (c) incremental where a new data generating mechanism incrementally replaces the existing mechanism (i.e. each example is a weighted average of the two mechanisms but over time the impact of the initial mechanism disappears)  d) reoccurring drift when a previously seen data distribution reappears \citep{survey},\citep{survey2}.

CD detection algorithms are categorized based on the kind of the statistics they apply \citep{survey}. The first category is the \textquoteleft Error rate-based algorithms\textquoteright, which monitor increases or decreases in the online error rate, if these changes are statistically significant, a drift alarm is raised. The second and biggest category is the \textquoteleft Data Distribution-based\textquoteright; here the algorithms quantify the dissimilarity between the historical data and the new data. A distance function is used to measure the dissimilarity between the two distributions and a statistical hypothesis test with respect to a significance level determines if a CD occurs. The last category \textquoteleft Multiple Hypothesis Test\textquoteright, applies similar techniques to the ones mentioned above but employs multiple hypothesis tests to determine the presence of a CD. These can be divided into two groups: parallel hypothesis tests and hierarchical multiple hypothesis tests; for more information refer to \citep{survey}. In this study our CD detection algorithm belongs to the second category. Importantly, to the best of our knowledge, none of the existing detection algorithms in this or any other category offers valid probabilistic guarantees, except for methods that rely on Conformal Martingales (CM) which belong to Conformal prediction framework. 
    
Conformal prediction offers a  framework for uncertainty estimation in machine learning, providing predictions with  confidence levels. Introduced in \cite{vovk:alrw}, conformal prediction leverages the exchangeability of data points to yield statistically valid predictions under a specified significance level. This approach is particularly valuable in dynamic environments characterized by CD, as it allows for adaptive model recalibration based on the reliability of predictions. Here we  employ  conformal prediction techniques, particularly through the use of ICM \cite{Volk:icm}. A comprehensive resource list, including videos, tutorials, and open-source libraries on conformal prediction,  can be found in the GitHub repository ``awesome-conformal-prediction" is an excellent starting point \citep{manokhin_2022_6467205}.

This paper extends our previous work. In \cite{pmlr-v179-eliades22a}, we introduced a novel betting function integrated with the ICM for rapid detection of CD points, enabling timely model retraining and accuracy restoration. Furthering this approach in \cite{pmlr-v204-eliades23a}, we enhanced change point detection using an ensemble of 10 diverse classifiers, each coupled with ICM.  This methodology trains different models on distinct sets of instances and applies ICM independently to the p-value sequences of each model. Consequently this minimizes instances with unavailable labels, even during false alarms. Additionally, if a model lags in identifying a change, the rest of the ensemble mitigates this delay, maintaining overall accuracy. In this study, we extend the ICM ensemble approach by introducing new  betting functions to improve accuracy. 

Overall, in this study, we present three contributions in the field of CD detection, building upon our previous findings:

1. Extension of the Cautious Betting Function: Following our initial work outlined in \cite{pmlr-v179-eliades22a}, we have advanced the cautious betting function by employing multiple probability density estimators. This enhancement allows for a more nuanced analysis of data streams, improving the detection accuracy of concept drift.

2. Advancement in Density Estimation Methods: Whereas in \cite{pmlr-v179-eliades22a}, we utilized a simple histogram density estimator, our current research extends the methodology by employing for the first time within the ICM framework a smoothed histogram and  a nearest neighbors density estimator. 

3. Ensemble of ICMs:   We use an ICM ensemble following our previous work in  \cite{pmlr-v179-eliades22a}; additionally we provide a detailed analysis of the ensemble performance and number of available predictions with different combinations of classifiers.

The rest of the paper starts with an overview of related work on addressing CD in Section \ref{sec:rw}. Section \ref{sec:cm} gives a brief overview of the ideas behind ICM. Section  \ref{sec:NCMs} describes the proposed approach. While Section \ref{sec:res} presents the experimental setting and performance measures used in our evaluation and reports the obtained experimental results. Finally, Section \ref{sec:concl} gives our conclusions and plans for future work. 

\section{Related Work}
\label{sec:rw}

This section offers an overview of the literature on CD and CM. Subsection \ref{sub:Concept_Drift} delves into the related work on CD, while Subsection \ref{sub:CM} explores the contributions and existing research in the field of CM.

\subsection{Concept Drift}\label{sub:Concept_Drift}

This subsection examines various research contributions relevant to CD. Given the vast amount of research on this topic we will present only the most prominent works related to the method we follow. Our discussion begins with an exploration of two comprehensive surveys in the field.

An extensive overview with over 130 high-quality publications is presented in \cite{survey}, highlighting key developments in the field that contribute to research related to CD. This survey  lists and discusses 10 popular synthetic datasets and 14 publicly available benchmark datasets, which are crucial for evaluating the efficacy of learning algorithms in environments where CD is prevalent.

Another  survey in \cite{survey2} delves into works specifically addressing CD. This survey presents an exhaustive study of both synthetic and real datasets that are publicly accessible and can be employed for benchmarking CD challenges. It also thoroughly reviews the various types of CD and the array of approaches devised to manage such changes in data streams.

As previously mentioned, CD detection algorithms can be  categorized into three groups based on the statistical methods they employ, as detailed in \cite{survey}. The first category encompasses `Error rate-based algorithms', which primarily focus on fluctuations in error rates as indicators of drift. The second and most extensive category is `Data Distribution-based' algorithms, which analyze shifts in data distributions. The final category is `Multiple Hypothesis Test' algorithms, which use a series of statistical tests to detect drift. In the following paragraphs, we present notable works corresponding to each of these categories.

\subsubsection{Error rate based methods}
A significant contribution to error rate-based methods for CD addressing is the Accuracy Weighted Ensemble (AWE) method proposed in \cite{awe}. This ensemble-based approach assigns weights to base classifiers based on their classification error, enhancing decision accuracy.

Another innovative approach is the Accuracy and Growth rate updated Ensemble (AGE) method suggested in \cite{ENS_AGE}. AGE is a hybrid technique that merges the strengths of single classifier and ensemble methods, utilizing the geometric mean of weights and growth rates of models. This design offers improved adaptability to various types of CDs, and experimental results have generally shown AGE's superior accuracy over its competitors.

The authors of this work \cite{dwm:2007} demonstrate the Dynamic Weighted Majority (DWM) method, an innovative ensemble approach tailored for addressing CD. The essence of DWM lies in its dynamic mechanism that involves creating, weighting, or even removing online learners depending on their performance. To demonstrate the effectiveness of this method in adapting to evolving data streams, the authors integrated it with Naive Bayes, resulting in a variant known as DWM-NB. This implementation showcases the method's capacity to adjust seamlessly to changes in data stream characteristics.

The Learn++.NSE algorithm, an ensemble-based approach presented in \cite{comp:sea}, is designed for incremental learning in nonstationary environments. It generates a new classifier for each data batch and combines these classifiers using a time-adjusted error-based dynamic weighted majority voting. Learn++.NSE is versatile, accommodating various drift types such as constant or variable rate of drift, concept class changes, and cyclical drifts.

The authors of \cite{comp:elec_sea} introduced a novel approach to handling data streams using a particle filter-based learning method, known as PF-LR, for evolving logistic regression models. This innovative method operates by selectively choosing the most effective step particles in maximizing classification accuracy for each current data batch. Demonstrating its robustness, PF-LR has delivered promising results, showing notable performance even with relatively small batch sizes. Its efficacy was further validated through extensive testing on both synthetic and real datasets, where it successfully outperformed various other state-of-the-art stream mining algorithms.

The Adaptive Random Forest (ARF) algorithm \cite{Gomes2017AdaptiveRF} is specifically tailored for evolving data streams. It demonstrates that ARF handles CDs effectively without requiring complex adjustments, using a resampling method and adaptive operators. Its performance in both parallel and serial implementations has proven efficient and accurate, making it a strong alternative in data stream scenarios.

Lastly, the authors of \cite{Classification_comparisons} proposed the Diversified Dynamic Class Weighted (DDCW) ensemble model. This model integrates a dynamic class weighting scheme and focuses on ensemble diversity. Tests on various real-world and synthetic datasets demonstrated its predictive capability and computational efficiency, comparing favourably against other adaptive ensemble methods.

In the landscape of error rate-based methods for CD detection, various approaches showcase unique strategies to adapt to evolving data streams dynamically. They offer practical solutions for various types of CD, demonstrating robust performance in synthetic and real-world datasets.
However, it is crucial to note that while these methods are adaptable and often practical, they generally differ from traditional statistical approaches grounded in probabilistic guarantees. Typically, error rate-based methods involve setting thresholds on drift detection metrics and may, to some extent, rely on assumptions about the distribution of error rates. Unlike entirely distribution-free methods, this aspect can introduce certain limitations in their adaptability and applicability.
\subsubsection{Data distribution Based}

The EDTC algorithm, introduced by the authors of \cite{comp:sea_stagger}, revolutionizes concept drift detection with an incremental ensemble methodology built on Ensemble Decision Trees. It uniquely incorporates random feature selection and a dual-threshold detection system, based on Hoeffding Bounds, to effectively differentiate concept drifts from noise. Extensive evaluation on both synthetic and real streaming datasets confirms EDTC's superior performance over traditional online algorithms, showcasing its efficacy in managing noisy data environments with exceptional precision.

The research presented in \cite{Reg_dist} introduces the Local Drift Degree (LDD) concept. LDD is a measurement designed to quantify regional density discrepancies between distinct sample sets, thereby identifying areas where density has increased, decreased, or remained stable. The authors also developed two algorithms, LDD-DIS and LDD-DSDA, which leverage LDD to manage CD effectively. LDD-DIS is adept at continuously monitoring regional density changes to pinpoint drifted instances, while LDD-DSDA employs LDD for both drift identification and model adaptation. Their experimental results on three benchmark datasets indicate enhanced accuracy over other methods.

Another noteworthy contribution is from the authors of \cite{comp:stagger_elec}, who proposed an incremental least squares density difference (LSDD) change detection method. This method assesses distributional differences using two non-overlapping windows, and its efficacy was validated on six synthetic datasets and one real-world dataset.

In our study presented in \cite{pmlr-v152-eliades21a}, we explored the integration of  ICM with a histogram betting function. This novel combination is specifically designed to detect violations of the EA and, as a result, identify CD in data streams. Notably, our approach is distribution-free, distinguishing it from other methods that often presuppose a specific distribution in their drift detection metrics. This aspect of our methodology aligns with addressing the open question mentioned in \cite{survey} regarding reliance on assumed distributions.

Continuing this theme, our subsequent work \cite{pmlr-v204-eliades23a} employs an ICM ensemble learning approach to tackle CD in data-stream classification effectively. This system comprises 10 classifiers, each trained on distinct data sizes and operating within a majority voting framework for making predictions. By analyzing unique p-value sequences generated by each classifier through ICM, our method efficiently detects change points, triggering retraining of the affected classifier. Tested on four benchmark datasets it demonstrates accuracy that matches or surpasses that of two state-of-the-art algorithms.

\subsubsection{Multiple Hypothesis Tests}

The Just-in-Time (JIT) approach, implemented by the authors of \cite{jit}, stands out for its effectiveness in managing recurring concepts. JIT operates by identifying the specific concept to which incoming examples belong and maintains a pool of models, each representing different concepts. A critical insight from this approach is that leveraging information from past experiences significantly enhances the ability to handle CD.

The authors of \cite{hierarchical_mult} introduced Hierarchical Change-Detection Tests (HCDTs). This novel approach employs a two-layer hierarchical architecture to address CD. The first layer, focuses on detection, utilizing any low-complexity drift detection method to identify potential changes in the data-generating mechanism quickly. Once a change is detected, the second layer, dedicated to validation, is activated to conduct a more in-depth analysis of the recently acquired data, thereby helping to minimize false alarms. A distinctive feature of HCDTs is their ability to automatically reconfigure after confirming and validating a change, making them adept at recognizing deviations from the newly established data generation state. For the validation layer, the authors propose two strategies: one is to estimate the distribution of the test statistics through likelihood maximization, and the other is to adapt existing hypothesis tests, such as the Kolmogorov-Smirnov test or the Cramer-Von Mises test, to fit within the HCDTs framework.

It should be noted that while robust and powerful, multiple hypothesis tests for CD detection often rely on test statistics that depend on specific distribution assumptions. This dependency can introduce potential limitations to their drift detection precision if these assumptions are not met.

\subsection{Conformal Martingales}\label{sub:CM}

In this section, we explore the contributions of various researchers towards testing the EA using CM. A significant challenge in this domain has been the assumption of specific distributions for test statistics, a limitation  which can be addressed by the use of CM. As detailed in \cite{vovk:alrw}, CM tests the EA without relying on any  assumptions regarding the test statistics.

One notable contribution \cite{Vovk:exch} in this field involves a method for online exchangeability testing based on Conformal Prediction  and CM. This method  computes a sequence of p-values using conformal prediction in an online setting, where each new example's p-value is determined using both new and previously seen examples. Following this, a Betting Function  is applied to each p-value, and the product of these BF outputs forms the Martingale's value. When the Martingale's value \( M \) becomes sufficiently large, the EA can be rejected at a significance level of \( 1/M \). Consequently in a time series if at some point the EA is violated we have a change point. 

Further developing the concept introduced in \cite{Vovk:exch}, another study \cite{ho:ex} introduces a  CM that utilizes the   mixture betting function. The authors of this work  formulated two types of martingale tests: one based on the values of the martingale itself, and the other on the differences observed in successive martingale values. 

Extending these principles, another study \cite{fedorova:2012plugin} applied them to test the exchangeability of data in two datasets, USPS and Statlog Satellite. The approach involves online testing, where data is processed sequentially, and the CM value is computed to reject the EA. They used a kernel density estimator as a BF, showing superior performance to the simple mixture BF.

In addition, the authors of \cite{Volk:icm} introduced an Inductive version of CM for detecting changes in time series. This study uses the initial observations of the time sequence to train the underlying model, and all nonconformity scores are calculated via this model. The authors experimented with several BFs and found that the pre-computed kernel BF yields the most efficient results, evidenced by the lowest mean delay in their tests on synthetic datasets. Their findings are comparable with other methods like CUSUM, Shiryaev-Roberts, and Posterior Probability statistics.

In   \cite{DBLP:journals/corr/abs-1207-1379},  the authors use CM to identify concept changes in data streams by examining the EA. Their  approach, grounded on Doob's Maximal Inequality, establishes a robust framework for hypothesis testing within time-varying data environments. They rigorously tested their methodology on both synthetic and real-world datasets, showcasing its applicability and effectiveness in detecting concept changes.

A novel real-time martingale-based approach is proposed in \cite{ho2019} using Gaussian Process Regression (GPR) to predict and detect anomalous flight behaviour as data arrives sequentially. The authors implemented multiple CM tests to reduce the number of false alarms and the detection delay time, again utilizing the mixture BF for Martingale calculation.

To conclude, inspired by the studies reviewed in this section, our methodology uses ICM to address the CD problem due to its distribution free nature.

\section{Inductive Conformal Martingales}
\label{sec:cm}
In this section we describe the basic concepts of ICM and how our nonconformity scores and p-values are calculated.

\subsection{Data Exchangeability}\label{subsec:ex}
Let $(Z_1,Z_2,\dots)$ be an infinite sequence of random variables. Then the joint probability distribution $\mathbb{P}(Z_1,Z_2,\dots,Z_N)$  is exchangeable if it is invariant under any permutation of these random variables. The joint distribution of the infinite sequence $(Z_1,Z_2,\dots)$ is exchangeable if the marginal distribution
of $(Z_1,Z_2,\dots,Z_N)$ is exchangeable for every $N\in{\mathbb{N}}$. Testing if the data is exchangeable is equivalent to testing if it is independent and identically distributed (i.i.d.); this is an outcome of de Finetti's theorem \citep{Schervish}: any exchangeable distribution on the data is a mixture of distributions under which the data is i.i.d.  

\subsection{Exchangeability Martingale}\label{subsec:exm}
A test exchangeability Martingale is a sequence of random variables $(S_1,S_2,S_3,\dots)$  being equal to or greater than zero that keep the conditional expectation $\mathbb{E}(S_{n+1}|S_1,\dots,S_n)=S_n$. 

To give an idea of how a martingale works, consider a fair game where a gambler with infinite wealth follows a strategy based on the distribution of the events in the game.  The gain acquired by the gambler can be described by the value of a Martingale, specifically Ville's inequality \citep{piaggio1939} indicates that the probability of having high profit $(C)$ is small,  $\mathbb{P}\{\exists n: S_n\geq  C \}\leq 1/C$.

According to Ville's  inequality \citep{piaggio1939}  for the case of the EA, a large final value of the Martingale suggests rejection of the assumption with a significance level equal to the inverse of the Martingale value, i.e. a  Martingale value of 10 or 100 rejects the hypothesis of exchangeability at the $10\%$ or $1\%$ significance level, respectively.   

\subsection{Calculating Non-conformity Scores and  P-values  } \label{subsec:pcncm}

Let $\{z_1,z_2, \dots\}$ be a sequence of examples, where $z_i=(x_i,y_i)$
with $x_i$  an object given in the form of an input vector,
and $y_i$  the label of the corresponding input vector.
The CM approach generates a sequence of p-values corresponding to the given sequence of examples and then calculates the martingale as a function of these p-values. As mentioned in Section \ref{sect:introduction}, this work employs CM's computationally efficient inductive version. ICM uses the first $k$ examples   $\{z_1,z_2, \dots,z_k\}$ in the sequence to train a classification algorithm, which it then uses to generate the p-values for the next examples. Consequently, it starts checking for violations of the EA from example \(z_{k+1}\), focusing on the sequence \(\{z_{k+1}, z_{k+2}, \dots\}\).

Our aim is to examine how strange or unusual a new example $z_j\in \{z_{k+1}, z_{k+2}, \dots\} $ is compared to the training examples. To make this possible, we define a function $A(z_i, \{z_1, \dots,\ z_k\})$, where $i \in \{k+1\ \dots\}$, called a nonconformity measure (NCM) that assigns a numerical value $\alpha_i$ to each example $z_i$, called nonconformity score (NCS). The NCM is based on the trained underlying classification algorithm. The bigger the NCS value of an example, the less it conforms  with $\{z_1, \dots, z_k\}$ according to the underlying algorithm.

For every new  example $z_j$ we generate  the sequence  $H_j=\{\alpha_{k+1},\alpha_{k+2},\dots,\alpha_{j-1},\alpha_{j}\}$ to calculate its p-value. Note that the NCSs in $H_j$ are calculated with the underlying algorithm  trained on $\{z_1,z_2, \dots,z_k\}$. 
Given the sequence $H_j$  we can calculate the corresponding p-value ($p_j$)  of the new example $z_j$ with the function:
\begin{equation}\label{eq:pvalue}        
	p_j = \frac{|\{ \alpha_{i}  \in H_j | \alpha_{i} > \alpha_{j}\}|+U_j\cdot |\{ \alpha_{i}  \in H_j | \alpha_{i} = \alpha_{j}\}| }{j-k},
\end{equation}
where $\alpha_{j}$ is the NCS of the new example and   $\alpha_{i}$ is the NCS of the $i^{th}$ element in the example sequence set and $U_j$ is a random number from the uniform distribution (0,1). For more information, refer to \cite{Vovk:exch}.

\subsection{Constructing Exchangeability Martingales}\label{subsec:cm}
An ICM is an exchangeability test Martingale (see Subsection \ref{subsec:exm}), which is calculated as a function of p-values such as the ones described in Subsection \ref{subsec:pcncm}.   

Given a sequence of p-values $(p_1,p_2,\dots)$ the martingale $S_n$ is calculated as:  
\begin{equation}\label{eq:Sn}        
	S_n=\prod_{i=1}^{n}f_i(p_i)
\end{equation} 
where $f_i(p_i)=f_i(p_i|p_1,p_2,\dots,p_{i-1})$  is the betting function \citep{Vovk:exch}.

The betting function should satisfy the constraint: $\int_{0}^{1}f_i(p)dp=1, f_i(p)\geq0 $ and also the $S_n$ must keep the conditional expectation:  $\mathbb{E}(S_{n+1}|S_0,S_1,\dots,S_n)=S_n$.

The integral $\int_{0}^{1}f_i(p)dp$ equals to 1 because $f_i(p)$ is the p-values $(p_1,p_2,\dots,p_{i-1})$ density estimator. We also need to prove that $\mathbb{E}(S_{n+1}|S_0,S_1,\dots,S_n)=S_n$ under any exchangeable distribution.  
\begin{proposition}\label{prob:1}
If $\int_{0}^{1}f_i(p)dp=1$  then under any exchangeable distribution it holds:	\[
	\mathbb{E}(S_{n+1} \mid S_0, S_1, \ldots, S_n) = S_n
	\]
\end{proposition}

\begin{proof}[Proof of Proposition \ref{prob:1}]
	The integral \(\int_{0}^{1} f_i(p) \, dp\) equals 1 
	
	We will now prove that the conditional expectation is preserved under any exchangeable distribution:
\begin{equation}
	\begin{aligned}
		\mathbb{E}(S_{n+1} \mid S_0, S_1, \ldots, S_n) &= \int_{0}^{1} \prod_{i=1}^{n} f_i(p_i) \cdot f_{n+1}(p) \, dp \\
		&= \prod_{i=1}^{n} f_i(p_i) \cdot \int_{0}^{1} f_{n+1}(p) \, dp \\
		&= \prod_{i=1}^{n} f_i(p_i) = S_n
	\end{aligned}
\end{equation}
\end{proof}
Using equation (\ref{eq:Sn}), it is easy to show that $S_n=S_{n-1}\cdot f_n(p_n)$, which allows us to update the martingale online. 
Let us say that the value of $S_n$ equals M, then Ville's inequality \citep{piaggio1939} suggests that we can reject the EA with a significance level equal to $1/M$.  

Note that we can calculate equation (\ref{eq:Sn})  in the logarithmic scale to deal with precision issues.
\section{Proposed Approach}\label{sec:NCMs}

This section describes the proposed methodology. Initially, we introduce the CAUTIOUS betting function, which is employed by   ICM to test the EA. We present two theorems that support the design of the CAUTIOUS betting function. Subsequently, we describe the extension of the CAUTIOUS betting function to incorporate multiple density estimators. We then detail the two density estimators utilized in our study. To the best of our knowledge, this is the first time those two  Density Estimators have been used within an ICM framework. We then demonstrate how Concept Drift (CD) is detected using ICM and finally explain how the ICM ensemble manages CD.

\subsection{The Cautious Betting Function}\label{subsec:cau}
In this section, we describe the Cautious Betting Function defined as $h_n$ and discuss its extension when it is built on multiple density estimators.
 
\subsubsection{Motivation and Definition}\label{subsec:single_cau}
An issue of the  CM and ICM is that they might need a significant amount of time to recover from a value very close to zero \citep{Volk:icm}. This betting function avoids betting (i.e.\ $h_n=1$) when insufficient evidence is available to reject the EA, thus keeping the value of $S_n$ from getting close to zero and reducing the time needed to detect a CD. Theorem \ref{thm:1} shows that under a uniform distribution of p-values, any betting function diverging from constant unity leads to \(S_{\infty}\) identically equaling zero. Here, \(S_{\infty}\) denotes the limiting value of the product of sequential bets over an infinite timeline. Theorem \ref{thm:2} extends this, showing that for any sequence of betting functions converging uniformly to a function other than constant unity, \(S_{\infty}\) will also converge to zero. These theorems support the design of the Cautious Betting Function for its strategic avoidance of unnecessary bets, ensuring faster CD detection.

\begin{theorem}\label{thm:1}
When the distribution of the p-values is uniform then for any betting function $f$ other than $f=1$, it follows that $S_{\infty} \equiv 0$
\end{theorem}
 
\begin{proof}[Proof of Theorem \ref{thm:1}]
	
	Let \( P \sim U(0,1) \), and consider \( S_n = \prod_{i=1}^{n}f(p_i) \) with a betting function \( f \neq 1 \). We aim to show that \( \lim_{n \to \infty} S_n = 0 \). 
	
	First, we express \( S_n \) in terms of logarithms:
	\[
	\ln(S_n) = \ln\left(\prod_{i=1}^{n}f(p_i)\right) = \sum_{i=1}^{n}\ln(f(p_i))
	\]
	
	Applying the Law of Large Numbers and Jensen's Inequality to the concave function \( \ln(x) \), we get:
	\[
	\lim_{n \to \infty} \frac{1}{n}\sum_{i=1}^{n}\ln(f(p_i)) = \mathbb{E}[\ln(f(P))] \leq \ln(\mathbb{E}[f(P)]) < 0
	\]
	since \( \mathbb{E}[f(P)] \) is in the range (0,1) for \( P \sim U(0,1) \).
	
	Consequently, the probability that \( S_n \) exceeds any positive threshold \( \epsilon \) approaches zero as \( n \) increases:
	\[
	\mathbb{P}(S_n > \epsilon) = \mathbb{P}\left(\sum_{i=1}^{n}\ln(f(p_i)) > \ln(\epsilon)\right) = \mathbb{P}\left(\frac{1}{n}\sum_{i=1}^{n}\ln(f(p_i)) > \frac{\ln(\epsilon)}{n}\right) \to 0 \text{ as } n \to \infty
	\]
	
	Therefore, \( S_n \overset{\mathbb{P}}{\to} 0 \). Furthermore, by the Martingale Convergence Theorem, \( S_n \overset{a.s.}{\to} S_{\infty} \), implying \( S_n \overset{\mathbb{P}}{\to} S_{\infty} \). Due to the uniqueness of limits, \( S_{\infty} \equiv 0 \).
	
\end{proof}

 \begin{theorem}\label{thm:2}
 	When   of the distribution of the p-values is uniform,  for any sequence of betting functions \(f_i\) that converges uniformly to a function \(f\) other than \(f = 1\), it follows that \(S_{\infty} \equiv 0\).
 \end{theorem}

\begin{proof}[Proof of Theorem \ref{thm:2}]
	Let \( p_i \sim U(0,1) \) and consider the product \( S_n = \prod_{i=1}^{n}f_i(p_i) \), where each \( f_i \) is a betting function taking values in \((0,1)\) and converges uniformly to a betting function \( f \). 
	
	We have:
	\[
	\ln(S_n) = \ln\left(\prod_{i=1}^{n}f_i(p_i)\right) = \sum_{i=1}^{n}\ln(f_i(p_i))
	\]
	
	Now, considering the average of the logarithms:
	\[
	\frac{1}{n}\sum_{i=1}^{n}\ln(f_i(p_i))
	\]
	
	By the uniform convergence of \( f_i \) to \( f \), and assuming that the \( p_i \)'s are i.i.d. samples from a uniform distribution \( U(0,1) \), we can apply the Law of Large Numbers (LLN). The LLN implies that the sample average converges in probability to the expected value. Furthermore, by applying Jensen's Inequality, which states that for a concave function, like the natural logarithm, the inequality \( \mathbb{E}[\phi(X)] \leq \phi(\mathbb{E}[X]) \) holds, we have:
	\[
	\lim_{n \to \infty} \frac{1}{n}\sum_{i=1}^{n}\ln(f_i(p_i)) = \mathbb{E}[\ln(f(P))] \leq \ln(\mathbb{E}[f(P)]) < 0
	\]

	Therefore, as \( n \to \infty \), the probability that \( S_n \) exceeds any positive threshold \( \epsilon \) approaches zero:
	\[
	\mathbb{P}(S_n > \epsilon) = \mathbb{P}\left(\sum_{i=1}^{n}\ln(f_i(p_i)) > \ln(\epsilon)\right) = \mathbb{P}\left(\frac{1}{n}\sum_{i=1}^{n}\ln(f_i(p_i)) > \frac{\ln(\epsilon)}{n}\right) \to 0 \text{ as } n \to \infty
	\]
	
Therefore, \( S_n \overset{\mathbb{P}}{\to} 0 \). Furthermore, by the Martingale Convergence Theorem, \( S_n \overset{a.s.}{\to} S_{\infty} \), implying \( S_n \overset{\mathbb{P}}{\to} S_{\infty} \). Due to the uniqueness of limits, \( S_{\infty} \equiv 0 \).
\end{proof}

Before defining the mathematical formulation of our Cautious Betting Function, let us consider the strategic interplay between two hypothetical players in a game of probability. These players symbolize the dynamic decision-making process inherent in our model. Player 1, evaluates the performance of Player 2, who employs a variable betting strategy based on the density estimator \(f_n\) (Interpolated Histogram or Nearest Neighbours). This evaluation guides Player 1's decision to bet or abstain. The following equation formalizes this strategic evaluation, where the decision to bet hinges on a comparison of recent and past performance metrics:

\begin{equation} \label{eq:hn}
	h_n(x)=\begin{cases}
		1 \quad &\text{if} \, S1_{n-1}/\underset{k}{\min} S1_{n-k}\leq\epsilon \\
		f_n \quad &\text{if} \, S1_{n-1}/\underset{k}{\min} S1_{n-k}>\epsilon  \\
	\end{cases}\\
\end{equation} 
with \(S1_n=\prod_{i=1}^{n}f_i(p_i)\) representing the cumulative product of betting functions applied to p-values, where $k\in\{1,\dots,\min(W,n-1)\}$, the parameters \(\epsilon > 0\) and $W \in \mathbb{N}$. \(\epsilon\), set to $100$, serves as a  threshold, in which betting or not is justified based on the evidence against the exchangeability assumption. A higher \(\epsilon\) implies a more cautious approach, requiring stronger evidence for betting, thus enhancing the model's precision by betting only when substantial evidence is present. Meanwhile, \(W\), fixed at $5000$, determines the amount of historical data considered,  to guide current betting decisions. Theorem \ref{thm:3}  shows than $h_n$ is a betting function.

\begin{theorem}\label{thm:3}
	$h_n$ is a betting function if $f_n$ is a betting function or equivalently a probability distribution function.
\end{theorem}
 \begin{proof}[Proof of Theorem \ref{thm:3}]
	
	$h_n$ is obviously always non negative. We also have to  show  that it integrates to 1:
	\begin{equation}
		\int_{0}^{1}h_n(p)dp=\begin{cases}
			\int_{0}^{1}1dp=1 \quad &\text{if} \, S1_{n-1}/\underset{k}{\min} S1_{n-k}\leq\epsilon \\
			\int_{0}^{1}f_n(p|p_1,p_2,\dots,p_{n-1})=1 \quad &\text{if} \, S1_{n-1}/\underset{k}{\min} S1_{n-k}>\epsilon  \\
		\end{cases}\\
	\end{equation} 
	
\end{proof}

We can now make our two player illustration more specific based on equation \ref{eq:hn}. Player 2's strategy is based on the density estimator  $f_n$. Player 1 observes Player 2's performance in the game during a time window $[n-k,n-1]$ and calculates the maximum profit Player 2 could have made by starting to play when $S1_i$ was at its minimum, i.e., $S1_{n-1}/\underset{k}{\min} S1_{n-k}$. If Player 2 incurs losses or has very small profits, then Player 1 opts not to bet, setting  $h_n$ to 1. However, if Player 2 could make a profit greater than a certain threshold, Player 1 bets by setting $h_n=f_n$.  

\subsubsection{Extension to Multiple Density Estimators}
Building on the two-player scenario, we extend our model to a multiplayer context, where multiple players with unique strategies impact decision-making. By incorporating multiple players, each employing distinct density estimators \(f_n^j\). The cautious betting function $h_n$ for this scenario  is defined as follows:

\begin{equation} \label{eq:hn-multi}
	h_n(x)=\begin{cases}
		1 \quad &\text{if} \, \max\limits_{j\in \{2, \dots, M\}} \left( \frac{S1_{n-1}^j}{\underset{k}{\min} S1_{n-k}^j} \right) \leq \epsilon \\
		f_n^m \quad &\text{otherwise, where } m = \argmax\limits_{j\in \{2, \dots, M\}} \left( \frac{S1_{n-1}^j}{\underset{k}{\min} S1_{n-k}^j} \right)  \\
	\end{cases}
\end{equation}

In  Equation (\ref{eq:hn-multi}),  $k$ ranges within $\{1,\dots,W\}$, where $W$  is defined as any integer from $1$ to $n-1$. The parameter $\epsilon$ is a positive threshold value. $M$  denotes the total number of players being observed, where in this study, we use ${1,3,6}$. For each player  $j$, $S1_n^j$ is calculated as the product $\prod_{i=1}^{n}f_i^j(p_i)$, where  $f_i^j$ represents the betting function or density estimator for that player. Finally, $f_n^m$ corresponds to the betting function or density estimator of the player who achieves the maximum ratio of $\frac{S1_{n-1}^j}{\underset{k}{\min} S1_{n-k}^j}$, identifying them as the most successful player within the time window $[n-k,n-1]$. Theorem \ref{thm:4}  shows than $h_n$ is a betting function.
\begin{theorem}\label{thm:4}
	$h_n$ is a betting function iff for every m $f_n^m$ is a betting function or equivalently a probability distribution function.
\end{theorem}

\begin{proof}[Proof of Theorem \ref{thm:4}]
	\[
	\int_{0}^{1}{h_n(p)dp}=\begin{cases}
		\int_{0}^{1}{1dp}=1 & \text{if } \max\limits_{j\in \{2, \dots, M\}} \left( \frac{S1_{n-1}^j}{\underset{k}{\min} S1_{n-k}^j} \right) \leq \epsilon, \\
		\int_{0}^{1}{f_n^m(p|p_1,\dots,p_{n-1})dp}=1 & \text{else, where } m = \argmax\limits_{j\in \{2, \dots, M\}} \left( \frac{S1_{n-1}^j}{\underset{k}{\min} S1_{n-k}^j} \right).
	\end{cases}
	\]

\end{proof}

Player 1 analyzes the performance of each player $j$ within the same time window \([n-k,n-1]\), calculating the maximum potential profit each could have achieved if they had started playing when their respective \( S1_i \) values were at their minimums. Player 1 then assesses whether any of these players have made significant profits exceeding a predefined threshold. If all players incur losses or only garner minimal profits, Player 1 decides not to bet, maintaining $h_n=1$. Conversely, if any player demonstrates a profit above the threshold, Player 1 adopts the strategy of the most successful player, setting $h_n$ to that player's corresponding  $f_n^m$. This is summarized in Equation (\ref{eq:hn-multi}).

\subsection{Density estimators}\label{subsec:de}
Here we introduce and describe the density estimators we have used to combine with the Cautious betting function namely the: Interpolated Histogram Estimator and the Nearest Neighbor Density Estimator.

\subsubsection{Interpolated Histogram Estimator}\label{subsec:hist}

This density estimator is initially calculated similarly to a standard histogram. The p-values $p_i \in [0,1]$, so we partition $[0,1]$ into a predefined number of equal bin intervals $\kappa$ and calculate the frequency of the observations in each bin. The histogram estimator is obtained by dividing these frequencies by the total number of observations and then multiplying by the number of bins.

We define the bins as follows: 
\[ B_1 = [0, 1/\kappa), B_2 = [1/\kappa, 2/\kappa), \dots, B_{\kappa-1} = [(\kappa-2)/\kappa, (\kappa-1)/\kappa), B_{\kappa} = [(\kappa-1)/\kappa, 1]. \]

For a p-value $p_n$ belonging to bin $B_j$, the density estimator is calculated as: 
\begin{equation}\label{eq:est} 
	\hat{f}_n(p_n) = \frac{n_j \cdot \kappa}{n-1},
\end{equation} 
where $n-1$ is the number of p-values observed so far and $n_j$ is the count of p-values in $B_j$. 

In cases where $n$ is small and there are empty bins (i.e., $\exists x : \hat{f}_n(x) = 0$), we reduce the number of bins $\kappa$ by 1. This reduction is repeated until there are no empty bins (i.e., $\nexists x : \hat{f}_n(x) = 0$).

To create an interpolated histogram, we then proceed as follows:

1. \textbf{Calculate Bin Centers}: Determine the center of each bin. The center $c_j$ of bin $B_j$ is given by $c_j = \frac{(2j-1)}{2\kappa}$ for $j = 1, \dots, \kappa$.

2. \textbf{Interpolate}: Apply an interpolation method (such as linear, spline, or polynomial interpolation) to these bin centers and their corresponding density estimator values $\hat{f}_n(c_j)$.

3. \textbf{Create Smoother Curve}: This interpolation creates a continuous curve that represents a smoother estimate of the density, providing a  continuous view of the data distribution compared to the discrete histogram.

Given a histogram with bins $B_j$ for $j = 1, \dots, \kappa$, the center $c_j$ of each bin is calculated as:
\begin{equation}\label{c_j}
	c_j = \frac{(2j-1)}{2\kappa}, \quad j = 1, \dots, \kappa.
\end{equation}

The density estimator for each bin center is given by:
\begin{equation}\label{f_n(c_j)}
	\hat{f}_n(c_j) = \frac{n_j \cdot \kappa}{n-1},
\end{equation}
where $n_j$ is the count of p-values in bin $B_j$, and $n-1$ is the total number of p-values observed so far.

To create the interpolated histogram, an interpolation method is applied to the set of points $\{(c_j, \hat{f}_n(c_j))\}_{j=1}^{\kappa}$. If we use linear interpolation, for example, the interpolated value at any point $x \in [c_1, c_k]$ can be approximated as:
\[
\hat{f}_n^{interp}(x) \approx \text{LinearInterpolation}\left(\{(c_j, \hat{f}_n(c_j))\}_{j=1}^{\kappa}, x\right).
\]
It's important to note that there are specific intervals within $[0,1]$ where linear interpolation is not applicable. In these intervals, the function $\hat{f}_n(x)$ takes constant values based on the nearest boundary points.

\begin{lemma}\label{eq:int_est} 
	The function \(\hat{f}_n^{interp}(x)\) for \(x \in [0, 1]\) can be defined as:
	\[
	\hat{f}_n^{interp}(x) = 
	\begin{cases}
		\hat{f}_n(c_1) & \text{if } x \in [0, c_1] \quad \text{(no interpolation)}, \\
		\left( \frac{\hat{f}_n(c_j) - \hat{f}_n(c_{j-1})}{c_j - c_{j-1}} \right) \cdot (x - c_{j-1}) + \hat{f}_n(c_{j-1}) & \text{if } x \in [c_{j-1}, c_j], \; j \in \{2, \ldots, k\}, \\
		\hat{f}_n(c_k) & \text{if } x \in [c_k, 1] \quad \text{(no interpolation)}.
	\end{cases}
	\]
\end{lemma}

\begin{proof}
	Let \((c_j, \hat{f}_n(c_j))\) and \((c_{j-1}, \hat{f}_n(c_{j-1}))\) be two points. The slope \(m\) of the line passing through these points is given by:
	\[ m = \frac{\hat{f}_n(c_j) - \hat{f}_n(c_{j-1})}{c_j - c_{j-1}} \]
	The y-intercept \(b\) can be found using the point-slope form:
	\[ b = \hat{f}_n(c_j) - m \cdot c_j \]
	Substituting the value of \(m\) into \(b\):
	\[ b = \hat{f}_n(c_j) - \left( \frac{\hat{f}_n(c_j) - \hat{f}_n(c_{j-1})}{c_j - c_{j-1}} \right) \cdot c_j \]
	Therefore, the equation of the line for linear interpolation between the points is:
	\[ y = \left( \frac{\hat{f}_n(c_j) - \hat{f}_n(c_{j-1})}{c_j - c_{j-1}} \right) x + \left( \hat{f}_n(c_j) - \left( \frac{\hat{f}_n(c_j) - \hat{f}_n(c_{j-1})}{c_j - c_{j-1}} \right) \cdot c_j \right) \]
	Alternatively, this can be expressed as:
	\[ y = \left( \frac{\hat{f}_n(c_j) - \hat{f}_n(c_{j-1})}{c_j - c_{j-1}} \right) \cdot (x - c_{j-1}) + \hat{f}_n(c_{j-1}) \]
	\qedhere
\end{proof}
This results in a smoother curve $\hat{f}_n^{interp}(x)$, providing a continuous representation of the density estimate across the interval $[0,1]$, except in the specified intervals where no interpolation occurs. Lemma \ref{lemma:2} shows that this betting function integrates to 1.

\begin{lemma}\label{lemma:2}
	The function $\hat{f}_n^{interp}(x)$ integrates to $1$.
\end{lemma}

\begin{proof}{Proof of Lemma \ref{lemma:2}}
	Recall from equations \ref{c_j} and \ref{f_n(c_j)} that 
	\[
	c_j = \frac{(2j-1)}{2\kappa}, \quad \hat{f}_n(c_j) = \frac{n_j \cdot \kappa}{n-1}, \quad \text{for } j = 1, \dots, \kappa.
	\]
	Thus, the integral of $\hat{f}_n^{interp}(x)$ over $[0, 1]$ is given by:
	\begin{align*}
		\int_0^1 \hat{f}_n^{interp}(x) \, dx &= \int_0^{c_1} \hat{f}_n(c_1) \, dx + \sum_{j=2}^\kappa \int_{c_{j-1}}^{c_j} \left(\frac{\hat{f}_n(c_j) - \hat{f}_n(c_{j-1})}{c_j - c_{j-1}} (x - c_{j-1}) + \hat{f}_n(c_{j-1})\right) \, dx \\ &\quad + \int_{c_\kappa}^1 \hat{f}_n(c_\kappa) \, dx.
	\end{align*}
	It is straightforward to see that $\int_0^{c_1} \hat{f}_n(c_1) \, dx$ and $\int_{c_\kappa}^1 \hat{f}_n(c_\kappa) \, dx$ represent the areas of two rectangles with dimensions $c_1$ and $\hat{f}_n(c_1)$, and $1 - c_\kappa$ and $\hat{f}_n(c_\kappa)$, respectively.
	
	Also, for $j = 2, \dots, \kappa$, the integral 
	\[
	\int_{c_{j-1}}^{c_j} \left(\frac{\hat{f}_n(c_j) - \hat{f}_n(c_{j-1})}{c_j - c_{j-1}} (x - c_{j-1}) + \hat{f}_n(c_{j-1})\right) \, dx
	\]
	is equal to the area of a trapezoid with height $\frac{1}{\kappa}$ and bases $\hat{f}_n(c_{j-1})$ and $\hat{f}_n(c_{j})$.
	
	The total area of the two rectangles is:
	\[
	A_1 + A_\kappa = \frac{1}{2\kappa} \cdot \frac{n_1 \cdot \kappa}{n-1} + (1 - \frac{2\kappa - 1}{2\kappa}) \cdot \frac{n_\kappa \cdot \kappa}{n-1} = \frac{n_1 + n_\kappa}{2(n-1)}
	\]
	The area of each trapezoid is:
	\[
	A_k = \frac{\left(\hat{f}_n(c_{j-1}) + \hat{f}_n(c_j)\right) \cdot \frac{1}{\kappa}}{2} = \frac{n_{j-1} + n_j}{2(n-1)}
	\]
	Summing all areas,
	\[
	A = \frac{n_1 + n_\kappa}{2(n-1)} + \sum_{j=2}^\kappa \frac{n_{j-1} + n_j}{2(n-1)} = 1
	\]
	since $\sum_{j=1}^\kappa n_j = n - 1$.
\end{proof}

\subsubsection{Nearest Neighbor Density Estimator}
The nearest neighbor density estimator, specifically the k-nearest neighbors (kNN) density estimator, is a non-parametric method used in statistical analysis for estimating the probability density function   of a random variable in a one-dimensional space, such as an interval. This method, as described below, is taken from \cite{KNN_Density_Estimation}.

Consider a set of \( N \) i.i.d. samples \( X_1, X_2, \ldots, X_N \) from a distribution within the interval \([0,1]\) with an unknown pdf \( f: \mathbb{R} \rightarrow \mathbb{R} \). The kNN density estimator aims to estimate the pdf at a point \( x \) in this interval. The estimation involves counting the fraction of samples within an interval centered at \( x \) with radius \( R_k(x) \), where \( R_k(x) \) is the distance to the \( k \)-th nearest neighbor of \( x \), and normalizing by the length of this interval and the total number of samples.

The length of the interval \( L(R_k(x)) \) is determined as follows, considering the boundaries of the interval \([0,1]\):
\begin{equation}
	L(R_k(x)) = \min(1 - x, R_k(x)) + \min(x, R_k(x))
\end{equation}

The estimator in the one-dimensional case is given by:
\begin{equation} \label{NN_estimator}
	\hat{f}(x) = \frac{\min(n-1,k-1)}{nL(R_k(x))}
\end{equation}
where \( \hat{f}(x) \) is the estimated density at point \( x \), \( k \) is the number of nearest neighbors, \( n \) is the total number of data points, and \( L(R_k(x)) \) is defined as above.

The choice of \( k \), the number of nearest neighbors, is critical. A smaller \( k \) provides a more localized estimate, while a larger \( k \) leads to a smoother estimate. Despite its simplicity, the estimator is effective in one-dimensional spaces and provides an approximately unbiased estimate of the density.

\subsection{Detecting CD using ICM}\label{subsec:CD_det1}
In order to detect a CD at a pre-specified significance level $\delta$, the Martingale value must exceed 1/$\delta$, which leads to the rejection of the EA. This process is summarized in Algorithm \ref{alg:1}. Specifically, if the Martingale value $S_k$ at a given point $k$ exceeds 100, a CD is detected with a significance level of 1\%, where $L$ denotes the number of p-values that our estimator uses.

	\begin{algorithm}
		\caption{Detect CD using ICM}\label{alg:1}
		\begin{algorithmic}[1]
			\Require Training set $\{z_1, z_2, \dots, z_k\}$, Test set $\{z_{k+1}, \dots, z_n\}$, significance level $\delta$
			\State Initialize $S_1 = 1$
			\For{$i = 1$ to $n - k$}
			\State $\alpha_i = A(z_{k + i}, \{z_1, \dots, z_k\})$
			\State $p_i = \frac{\#\{j : \alpha_j > \alpha_i\} + U_i \#\{j : \alpha_j = \alpha_i\}}{i}$
			\State Calculate betting function $h_i = h(p_{i - L}, \dots, p_{i - 1})$
			\State $S_i = S_{i - 1} \cdot h_i(p_i)$
			\If{$S_i > \frac{1}{\delta}$}
			\State Raise an Alarm
			\EndIf
			\EndFor
		\end{algorithmic}
	\end{algorithm}

\subsection{Ensemble ICM }\label{subsec:ens_icm}
In our previous study \citep{pmlr-v179-eliades22a}, we used a single classifier and  ICM to detect CDs. When a CD was detected, we waited until a pre-specified number of instances arrived to be used as  training set. In our current methodology for addressing the CD problem, we train 10 classifiers using a predetermined number of instances $\theta = \{100, \dots, 1000\}$, due to computational constraints as it is very difficult to calculate an optimum value of $\theta$ and the ensemble size, since it depends on the characteristics of each dataset. Each classifier is trained with a different number of instances, and as new observations arrive, a distinct sequence of p-values is calculated for each classifier. We apply a different ICM to each sequence of p-values to detect  possible CDs. If a CD is detected, the affected classifier stops making predictions.  

Since a CD begins before it is detected, we  go backward in time to construct the new training set. In particular, the new training set consists of the $k$ instances starting from the point that the Martingale has exceeded a threshold $r$ (smaller than the detection threshold); note that, if necessary, we wait until a sufficient number of instances are observed.   Specifically, assuming that a CD occurs at an instance with timestamp $i$, we find the closest timestamp $d$ to $i$ in which the Martingale passes a value equal to $r = \{2, 10, 100\}$. We include instances with indices $\{d,\dots, d+k\}$ as the training set, where $d=\max\{j:S_j^\mathrm{classifier}<r\}+1$ and the classifier is retrained at point $d+k$. Adding instances in the training set that arrived before an alarm was raised allows for predictions for as many instances as possible.
This entire process is carried out in parallel for each classifier, and a majority vote is used to predict an instance. The entire algorithm is summarized in Algorithm \ref{alg:2}. Even in cases where the algorithm raises a false alarm and forces model retraining, the other classifiers can still provide predictions for the corresponding time interval. If the algorithm fails to detect a CD in the p-values produced by a classifier, this classifier might provide misleading predictions. However, since other classifiers have been retrained, they can counterbalance the predictions of the classifier that was not retrained, thereby ensuring that performance is maintained.

It should be noted that calculating the optimum value of $\theta$ and the number of classifiers is computationally intensive. We opted for a practical approach to overcome these constraints by selecting a range of delay values. Specifically, we considered $\theta = \{100, \dots, 1000\}$. Although this range may not guarantee maximum accuracy, it strikes a balance between computational efficiency and assessing the impact of delays on performance for the different datasets used in this study.

	\begin{algorithm}
		\caption{CD Detection using Multiple Classifiers}\label{alg:2}
		\begin{algorithmic}[1]
			\Require Instances $\{z_1, z_2, \dots, z_n\}$, significance level $\delta$, $L$, $\theta=[100,200,\dots,1000]$, $r$
			\State Initialize $Model\ is\ updated[1:10] = False$, $d[1:10]=1$
			\For{$i=1$ to $n$}
			\For{Classifier $=1$ to $10$}
			\If{$Model\ is\ updated[\mathrm{Classifier}] = true$}
			\State $\alpha^{\mathrm{Classifier}} = \alpha^{\mathrm{Classifier}} \cup A^{\mathrm{Classifier}}(z_{i},X_{\mathrm{Classifier}})$
			\State $U_i^{\mathrm{Classifier}} = rand()$ \Comment{a random number $U_i^{\mathrm{Classifier}} \in U(0,1)$}
			\State $p^{\mathrm{Classifier}}_i = \frac{\#\{j : \alpha^{\mathrm{Classifier}}_j > \alpha^{\mathrm{Classifier}}_\text{end}\} + U_i^{\mathrm{Classifier}} \#\{j : \alpha^{\mathrm{Classifier}}_j = \alpha^{\mathrm{Classifier}}_\text{end}\}}{\#a^{\mathrm{Classifier}}}$
			\State \Comment{Where $alpha^{\mathrm{Classifier}}_\text{end}$ is the last  element added in $alpha^{\mathrm{Classifier}}$}
			\State Calculate betting function $h_i = h(p_{i-L}^{\mathrm{Classifier}}, \dots, p_{i-1}^{\mathrm{Classifier}})$
			\State $S^\mathrm{Classifier}_i = S^\mathrm{Classifier}_{i-1} \cdot h_i(p_i^{Classifier})$
			\State $pred_i^{\mathrm{Classifier}}$ = prediction for instance $z_i$ from model $Classifier$
			\If{$S^\mathrm{classifier}_i > \frac{1}{\delta}$}
			\State $Model\ is\ updated[\mathrm{Classifier}] = false$
			\State $d[\mathrm{Classifier}] = \max\{i : S^\mathrm{classifier}_i < r\} + 1$
			\State  $\alpha^\mathrm{Classifier}=[]$;\Comment{reinitialize nonconformity vector for the specific Classifier} 
			\EndIf
			\ElsIf{$i - d[\mathrm{Classifier}] \leq \theta[Classifier]$}
			\State $X_{\mathrm{Classifier}} = \{z_{d[Classifier]}, \dots, z_{d[\mathrm{Classifier}] + \theta[\mathrm{Classifier}]}\}$
			\State Retrain Model($Classifier$) using $X_{\mathrm{Classifier}}$ as training set
			\State $Model\ is\ updated[Classifier] = true$
			\State Update nonconformity function $A^{\mathrm{Classifier}}$ using the Model($Classifier$)
			\State $S^\mathrm{Classifier}_i = 1$
			\EndIf
			\EndFor
			\State The predicted label $pred_i$ for instance $z_i$ is the mode of the predicted labels from the 10 classifiers.
			\EndFor
		\end{algorithmic}
	\end{algorithm}

\section{Experiments and Results}\label{sec:res}

In this section, we conduct an experimental evaluation of the proposed approach (Algorithm \ref{alg:2})  and compare its performance to that of existing methods. Specifically, we examine the performance improvement of the ICM ensemble using two synthetic datasets (STAGGER, SEA) and two real datasets (ELEC, AIRLINES). We compare the proposed approach to three state-of-the-art methods from the literature: the DWM method \citep{dwm:2007}, the AWE method \citep{awe} and the ARF method \citep{Gomes2017AdaptiveRF} .

\subsection{Datasets} \label{subsec:Dataset}

\subsubsection{Synthetic Benchmark Datasets}
The STAGGER dataset \citep{stagger:dataset} is a standard benchmark dataset for evaluating CD detection. To conduct our simulations, we generated one million instances consisting of three categorical attributes and a binary target output. The drift type we considered is sudden, with four different concepts. A drift occurs every 10,000 examples, which corresponds to the chunk size.

The SEA \citep{DATASET:SEA} dataset  is a popular synthetic dataset that contains sudden CD. We have generated 1,000,000 instances for our simulations, where each example is described by 3 numeric attributes and has a binary label. There are four concepts and drift occurs every 250,000 examples (i.e. each chunk consists of 250,000 examples).   In this dataset, the three variables take random values from the interval [0,10] and if the sum of the first two variables is less than or equal to a pre-specified threshold, then the instance is assigned to class 1 otherwise to 0, the third variable is irrelevant. It is worth noting that here we test the transition from concept $a\rightarrow b\rightarrow c\rightarrow d$ while in the case of STAGGER, we test the transition from concept $a\rightarrow b\rightarrow c\rightarrow d\rightarrow a,\dots$. It should be noted that in our study, we introduced noise to both datasets by changing 10\% of the labels. We evaluated the performance of our approach with and without noise to assess its robustness to noisy data.

\subsubsection{Real World Benchmark Datasets}

The ELEC dataset \citep{DATASET:ELEC} is a time series dataset containing 45,312 instances recorded at 30-minute intervals with a binary class label that indicates whether there has been a rise or a drop in price compared to a moving average of the last 24 hours. Each instance in the dataset consists of eight variables. The data has been collected from the Australian New South Wales Electricity Market. In our experiments, we excluded the variables related to time, date, day, and period and only used nswprice, nswdemand, transfer, vicprice, and vicdemand as predictors. We aimed to predict future values using a training set that consists of observations from less than one week.  Note that sorting the data in this dataset based on the time it was received is crucial.

The Airlines dataset, referenced in \citep{Airlines}, is a collection of flight arrival and departure records of commercial flights within the USA from October 1987 to April 2008. The data was gathered from various USA airports and contains 539,383 instances, each described by seven features. The main objective of this dataset is to classify whether a given flight was delayed or not.

\subsection{Experimental  Setting}

This section details the configurations used in our experiments, where we employed a single ICM as in \cite{pmlr-v179-eliades22a} and an ICM ensemble composed of 10 treebaggers. Each treebagger contains $40$ trees constructed using the bootstrap aggregating (bagging) technique. Specifically, subsets of the training data are randomly selected with replacement, and the resulting trees are combined using majority voting to form the final ensemble. For each example, the 10 classifiers provide prediction labels, and the final prediction is obtained through majority voting. Each classifier outputs the posterior probability $\tilde{p}_j$ for the true label $y_j$, and the NCM is defined as $\alpha_j=-\tilde{p}_j$. As mentioned before, a separate sequence of p-values is computed for each classifier, and we apply ICM on each produced sequence of p-values to detect CD. Once a CD is detected, the current classifier ceases predictions and is retrained after a sufficient number of observations have been collected. As in \cite{pmlr-v204-eliades23a} to determine the new training set for each classifier, as described in Algorithm \ref{alg:2} we use $\theta={100, 200, \dots, 1000}$ which refers to the number of instances used for training each classifier and $r={2,10,100}$ which is a threshold value used to determine when to start the formation of a new training set before the point where a CD is detected.  Specifically, we retrain the specific classifier when the Martingale value of the corresponding classifier exceeds a value $1/\delta$, where $\delta=0.01$ is the required significance level.

In our study, the Cautious Betting function is employed in combination with some variations of the Interpolated Histogram and Nearest Neighbor, resulting in the following betting functions:

a) Cautious-InterHist: This betting function combines the Cautious Betting approach with an interpolated Histogram density estimator(we applied interpolation between the centers of the bins) as presented in Proposition \ref{eq:int_est} featuring $15$ bins. It utilizes a single interpolated histogram to inform the betting function.

b) Cautious-MultiInterHist: This variant integrates the Cautious Betting function with three different interpolated histogram density estimators (Proposition~\ref{eq:int_est}), as specified in Equation~\ref{eq:hn-multi}. In this setup, \( M = 3 \), and each of the density estimators \( f_n^1, f_n^2, \) and \( f_n^3 \) corresponds to a distinct bin count of $5$, $10$, and $15$ bins, respectively.

c) Cautious-Multi(InterHist-NN): In this configuration, the Cautious Betting function is combined with three interpolated histogram density estimators and three nearest neighbor density estimators, as specified in Equation~\ref{eq:hn-multi}. In this setup, \( M = 6 \), with each density estimator identified as follows: $f_n^1, f_n^2$, \text{ and } $f_n^3$: Interpolated histogram density estimators (Proposition~\ref{eq:int_est}) with bin counts of $5$, $10$, and $15$ bins, respectively.  $f_n^4, f_n^5$, and  $f_n^6$: Nearest neighbor density estimators (Equation~\ref{NN_estimator}) with neighbor counts of 5, 10, and 15, respectively.

Throughout these experimental configurations, the parameters of the Cautious Betting function, denoted as $\epsilon$ and W, are consistently maintained at $100$ and $5,000$, respectively, as in \cite{pmlr-v179-eliades22a}.

In the forthcoming subsection, we will conduct simulations on two Synthetic and two Real-world datasets, focusing on evaluating the accuracy and the rate of available predictions. The presented results are averaged over five simulations on each dataset. Additionally, we compare the accuracy of the proposed approach with that of two state-of-the-art algorithms: AWE and DWM-NB, described in Section \ref{sec:rw}. The accuracies of these two algorithms were obtained from \cite{accur}.    
\subsection{Hypothesis Testing Procedure}\label{HT}

In our experiments, we applied a hypothesis testing procedure to compare the accuracy of the new model with the existing model. We first prove a lemma to help us perform our hypothesis test. Then, we show the required steps for implementing these hypothesis tests.

\begin{lemma}\label{lemma:3}
	Let \( x_i \) be a binary vector of size \( n \times 1 \), representing the outcomes of \( n \) independent Bernoulli trials with success probability \( p \) in the \( i \)-th simulation. Assume the following:
	\begin{enumerate}
		\item Each \( x_i \) is independently and identically distributed (i.i.d.) across trials within a simulation.
		\item The simulations are independent of each other, implying that \( x_i \) for different \( i \) are independent.
		\item The probability of success \( p \) remains constant across all trials and simulations.
	\end{enumerate}
	Let \( X_i = \frac{\sum_{j = 1}^{n} x_{i,j}}{n} \) denote the proportion of correctly classified instances by a classifier in the \( i \)-th simulation. Suppose the classifier's performance (success rate) is \( p \) in each simulation, and the number of simulations is \( k \). The overall mean of correct predictions across all simulations is given by
	\[
	X = \frac{\sum_{i = 1}^{k} X_i}{k}.
	\]
	Then, under these assumptions, the aggregate mean \( X \) follows a normal distribution with mean \( p \) and variance:
	\[
	\text{Var}(X) = \frac{p(1 - p)}{kn}.
	\]
\end{lemma}

\begin{proof}{proof of Lemma \ref{lemma:3}} 
	1. \textbf{Distribution of Individual Simulation Means \( X_i \):}
	Each binary vector \( x_i \) contains the results of \( n \) independent Bernoulli trials, each with a success probability of \( p \). Thus, the sample mean \( X_i \) follows a normal distribution by the central limit theorem (CLT) due to the sufficiently large number of trials \( n \):
	\[
	X_i \sim N\left(p, \frac{p(1 - p)}{n}\right).
	\]
	
	2. \textbf{Variance of Aggregate Mean \( X \):}
	The overall mean \( X \) is the mean of all \( k \) simulation means:
	\[
	X = \frac{\sum_{i = 1}^{k} X_i}{k}.
	\]
	By the properties of variances of independent random variables, the variance of \( X \) is:
	\[
	\text{Var}(X) = \frac{\sum_{i = 1}^{k} \text{Var}(X_i)}{k^2}.
	\]
	Since each \( X_i \) has a variance:
	\[
	\text{Var}(X_i) = \frac{p(1 - p)}{n},
	\]
	the variance of the overall mean is:
	\[
	\text{Var}(X) = \frac{k \cdot \frac{p(1 - p)}{n}}{k^2} = \frac{p(1 - p)}{kn}.
	\]
	Note that to reduce complexity across simulations, we assumed that n
	n (the number of available predictions) remains constant. In cases where an experiment yields different values of n for each simulation, we calculate and use their average.\\
	3. \textbf{Distribution of Aggregate Mean \( X \):}
	By the CLT and the above calculation, the aggregate mean \( X \) also approximates a normal distribution:
	\[
	X \sim N\left(p, \frac{p(1 - p)}{kn}\right).
	\]
	
\end{proof}

The estimated accuracy of the classifier (denoted as \(\hat{p}\)) is given by:
\[
\hat{p} = \frac{\sum_{i = 1}^{k} \sum_{j = 1}^{n} x_{i,j}}{kn},
\]
where:
\begin{itemize}
	\item \( k \) is the number of simulations.
	\item \( n \) is the average number of instances per simulation.
	\item \( x_{i,j} \) is a binary variable representing the correctness of the \( j \)-th prediction in the \( i \)-th simulation.
\end{itemize} 

Now assume that we have two classifiers with accuracies \( p \) and \( q \), respectively, where each classifier is tested on \( n_1 \) and \( n_2 \) instances respectively performing \( k_1 \) and \( k_2 \) simulations. We would like to test whether the new model is better than the old one. We now show the steps required for our hypothesis test.

\begin{itemize}
	\item \textbf{Step 1: Define the Hypotheses}
	\begin{itemize}
		\item \textbf{Null Hypothesis (H\textsubscript{0}):} The new model is not better than the old model:
		\[
		H_0: \hat{p} \leq \hat{q},
		\]
		where \(\hat{p}\) represents the mean accuracy of the new model and \(\hat{q}\) the mean accuracy of the old model.
		
		\item \textbf{Alternative Hypothesis (H\textsubscript{1}):} The new model is better than the old model:
		\[
		H_1: \hat{p} > \hat{q}.
		\]
	\end{itemize}
	
	\item \textbf{Step 2: Calculate the Observed Difference}
	Compute the observed difference in mean accuracies between the two models:
	\[
	d = \hat{p} - \hat{q}.
	\]
	
	\item \textbf{Step 3: Compute the Standard Error of the Difference}
	Given the variances of both models and the correlation coefficient \( \rho \), calculate the standard error (SE) of the difference:
	\[
	\text{SE}(X - Y) = \sqrt{\frac{\hat{p}(1 - \hat{p})}{k_1n_1} + \frac{\hat{q}(1 - \hat{q})}{k_2n_2} - 2\rho \sqrt{\frac{\hat{p}(1 - \hat{p})}{k_1n_1} \cdot \frac{\hat{q}(1 - \hat{q})}{k_2n_2}}}.
	\]
	
	\item \textbf{Step 4: Compute the Test Statistic}
	Calculate the test statistic (\( Z \)) using the observed difference and the standard error:
	\[
	Z = \frac{d}{\text{SE}(X - Y)}.
	\]
	
	\item \textbf{Step 5: Compare Against the Critical Value}
	Determine the critical value based on the desired confidence level:
	\begin{itemize}
		\item For a one-tailed test at a 95\% confidence level, the critical value is approximately \( Z_{0.95} = 1.645 \).
		\item If \( Z \geq 1.645 \), reject the null hypothesis and conclude that the new model is better than the old model.
		\item Otherwise, fail to reject the null hypothesis.
	\end{itemize}
\end{itemize}

\subsection{Evaluation of Betting Function Performance}

This subsection details our experiments across four datasets: STAGGER, SEA, ELEC, and AIRLINES. We utilized a simple ICM and a 10-ICM ensemble for our tests. For the simple ICM, the training set sizes were configured to 200, 1000, 300, and 200 instances, respectively, as aligned with our prior research \cite{pmlr-v179-eliades22a}. We abbreviate our betting functions as follows: IH (InterHist), MIH (MultiInterHist), MINN (MultiInterHist-NN), CAU (Cautious). CAU, employing a 15-bin histogram, is used in both the single classifier case as detailed in \cite{pmlr-v179-eliades22a} and in an ensemble setting as detailed in \cite{pmlr-v204-eliades23a}."

Table \ref{tab:acc_1cl} presents the accuracy of a single classifier employing different betting functions across various datasets,  noise levels and different values of $r$. Notably, the MIHNN, which integrates three interpolated histograms with three nearest neighbours, frequently surpasses or matches the performance of other functions. Although the performance of 3 nearest neighbours alone is not presented, it is essential to note that this method generally lowers performance and often fails to detect many changes. However, when it does detect a change, it does so with exceptional speed.

\begin{table}[h]
	\caption{Single Classifier Performance\label{tab:acc_1cl}. }
	
	\begin{tabular}{|l|l|l|l|l|l|l|}
		\hline
		\textbf{Dataset} & \textbf{Noise \%} & \textbf{r } & \textbf{IH} & \textbf{MIH} & \textbf{MIHNN}&\textbf{CAU} \\
		\hline
		STAGGER & 0 & 100 & 0.996 & 0.996 & 0.999 & 0.996 \\
		STAGGER & 0 & 10 & 0.996 & 0.996 & 0.999&- \\
		STAGGER & 0 & 2 & 0.997 & 0.997 & 0.999&- \\
		STAGGER & 10 & 100 & 0.947 & 0.947 & 0.947&0.944 \\
		STAGGER & 10 & 10 & 0.948 & 0.948 & 0.948&- \\
		STAGGER & 10 & 2 & 0.948 & 0.948 & 0.949&- \\
		\hline
		SEA & 0 & 100 & 0.982 & 0.982 & 0.981&0.982 \\
		SEA & 0 & 10 & 0.982 & 0.981 & 0.982&- \\
		SEA & 0 & 2 & 0.982 & 0.982 & 0.981&- \\
		SEA & 10 & 100 & 0.914 & 0.914 & 0.914&0.914 \\
		SEA & 10 & 10 & 0.914 & 0.914 & 0.914&- \\
		SEA & 10 & 2 & 0.914 & 0.914 & 0.913&- \\
		\hline
		ELEC & unknown & 100 & 0.751 & 0.753 & 0.762&0.748 \\
		ELEC & unknown & 10 & 0.758 & 0.758 & 0.758&- \\
		ELEC & unknown & 2 & 0.749 & 0.755 & 0.761&- \\
		\hline
		AIRLINES & unknown & 100 & 0.605 & 0.599 & 0.618&0.591 \\
		AIRLINES & unknown & 10 & 0.605 & 0.602 & 0.616&- \\
		AIRLINES & unknown & 2 & 0.607 & 0.602 & 0.617&- \\
		\hline
	\end{tabular}

\end{table}

Comparing IH with CAU reveals an improvement in accuracy, especially in real-world datasets, where substituting a 15-bin histogram with a 15-bin interpolated histogram in the Cautious function improves performance. In synthetic datasets like STAGGER and SEA, performance remains consistent. Comparing IH with MIH, we observe similar performances in synthetic datasets, regardless of the noise level. However, MIH slightly outperforms IH in real-world datasets such as ELEC and AIRLINES. MIHNN outperforms all other functions across all datasets, except in SEA, where there is no notable improvement and, in some cases, a reduction in accuracy.

In noise-free datasets, accuracy approaches the theoretical ideal of $1.0$, particularly in the STAGGER dataset. With a noise level of $10\%$, the accuracy closely aligns with the theoretical expectation of $0.95$. Below, we present the accuracy data for a single classifier with various betting functions across multiple datasets and noise levels. 

A decrease in $r$ should theoretically be beneficial in synthetic datasets because we begin taking observations before the drift is detected. This provides more instances available for making predictions that belong to the same distribution until a change occurs, thus improving accuracy because the ratio of instances that belong to the distribution that the model is trained on to the total number of instances until a new concept drift is detected is higher. However, this trend is observable on the STAGGER dataset in IH and MIH because MIHNN detects drift quickly. In the SEA dataset, it is very difficult to observe this conclusion because we only have 3 concept drifts in contrast to the STAGGER dataset, which has 99.

As $r$ decreases in the real-world datasets, such as the ELEC dataset, from 100 to 10, we observe an increase in accuracy for all betting functions. However, if we further decrease $r$ from 10 to 2, there is a decrease in accuracy. This is not surprising because, in contrast to synthetic datasets, real-world datasets may exhibit incremental or gradual drift until the distribution switches completely. Consequently, our model is trained on instances that differ from the final distribution, thus reducing accuracy.

In the case of the AIRLINES dataset, a decrease in $r$ from 100 to 10 improves accuracy for the IH and MIH functions, while for the MIHNN function, we observe a decrease. This is because MIHNN detects concept drift quickly and includes observations in the training set that do not belong to the final distribution. If we further reduce $r$ from 10 to 2 for the MIHNN function, there is no change in accuracy, while for the MIH function, accuracy slightly improves but remains worse than the accuracy obtained for $r=100$.

	Building on these insights, we now expand our analysis to an ICM ensemble of 10 classifiers, offering a comparative perspective against the single-classifier configuration discussed here.

 Table \ref{tab:Extended1A}, provides an  overview of our experimental results as above specifically we show  the performance levels attained by the ICM ensemble using various betting functions  and demonstrates how this performance fluctuates with different configurations of the $r$ parameter.

\begin{table}[h]
	\caption{ICM Ensemble Performance\label{tab:Extended1A}.}
	
	\begin{tabular}{|l|l|l|l|l|l|l|}
		\hline
		\textbf{Dataset} & \textbf{Noise \%} & \textbf{r } & \textbf{IH} & \textbf{MIH} & \textbf{MIHNN}&\textbf{CAU} \\
		\hline
		STAGGER & 0 & 100 & 0.995 & 0.996 & 0.999 & 0.995 \\
		STAGGER & 0 & 10 & 0.995 & 0.995 & 0.999 & 0.995 \\
		STAGGER & 0 & 2 & 0.996 & 0.998 & 0.999 & 0.996\\
		STAGGER & 10 & 100 & 0.948 & 0.947 & 0.949 & 0.947 \\
		STAGGER & 10 & 10 & 0.948 & 0.948 & 0.949 & 0.948\\
		STAGGER & 10 & 2 & 0.948 & 0.948 & 0.949 & 0.948\\
		\hline
		SEA & 0 & 100 & 0.981 & 0.981 & 0.982 & 0.983 \\
		SEA & 0 & 10 & 0.982 & 0.981 & 0.982 & 0.982\\
		SEA & 0 & 2 & 0.982 & 0.980 & 0.981 & 0.982\\
		SEA & 10 & 100 & 0.916 & 0.917 & 0.920 & 0.917 \\
		SEA & 10 & 10 & 0.916 & 0.916 & 0.920 & 0.915\\
		SEA & 10 & 2 & 0.916 & 0.915 & 0.920 & 0.916\\
		\hline
		ELEC & unknown & 100 & 0.765 & 0.766 & 0.766 & 0.765\\
		ELEC & unknown & 10 & 0.768 & 0.768 & 0.765 & 0.766\\
		ELEC & unknown & 2 & 0.767 & 0.766 & 0.767 & 0.767\\
		\hline
		AIRLINES & unknown & 100 & 0.649 & 0.651 & 0.648 & 0.635\\
		AIRLINES & unknown & 10 & 0.649 & 0.651 & 0.649 & 0.637\\
		AIRLINES & unknown & 2 & 0.649 & 0.651 & 0.648 & 0.637\\
		\hline
	\end{tabular}
\end{table}

While observing Table \ref{tab:Extended1A} and comparing IH with MIH, we observe similar performances in synthetic datasets. In the STAGGER dataset, MIH either ties with or outperforms IH in most cases, except for one instance. In the SEA dataset, results are mixed. For the ELEC dataset, IH and MIH perform similarly. However, in the AIRLINES dataset, MIH performs better than IH.

When we compare MIHNN with IH and MIH, we see that in the STAGGER dataset, MIHNN dominates and approximates the theoretical expectation of accuracy. In the SEA dataset without noise, the results are mixed, but when noise is injected, MIHNN performs better. For the ELEC dataset, MIHNN performs similarly to the other two betting functions. In the AIRLINES dataset, MIHNN performs worse than MIH and similarly to IH.

A reduction in the $r$ value does not seem to play a significant role in accuracy improvement due to the use of an ensemble. However, small changes can be seen when using the MIH betting function in the STAGGER dataset with noise, where a decrease in $r$ improves accuracy. In the ELEC dataset, setting $r$ to 10 improves performance when using IH and MIH, while MIHNN performance slightly reduces. Further reducing $r$ from 10 to 2 results in decreased performance for IH and MIH, while MIHNN performance improves.

For the AIRLINES dataset, a reduction of $r$ from 100 to 10 benefits IH and MIHNN, while MIH is not affected. A further decrease from 10 to 2 does not affect IH but slightly reduces MIHNN accuracy.

Because in Table \ref{tab:Extended1A}, the differences in accuracies appear minimal or are on the boundary for one betting function to be definitively better than another, we have performed a hypothesis test on these accuracies as described in subsection \ref{HT}. We assumed a perfect negative correlation of \(-1\) to make it more demanding to reject the null hypothesis. This assumption helps us avoid false positives. The results of these tests using r=10 are presented in Table \ref{tab:HT}, where \(p_1\), \(p_2\), \(p_3\) and \(p_4\) represent the accuracies of IH, MIH, MIHNN and CAU respectively. Note that we used precision with up to 15 decimal points while performing these hypothesis tests.
	
\begin{table}[h]
	\centering
	\caption{Summary of Hypothesis Test Results Across Datasets}
	\label{tab:HT}
	\begin{tabular}{|c|c|c|c|c|}
		\hline
		\textbf{  } & \textbf{STAGGER} & \textbf{SEA} & \textbf{ELEC} & \textbf{AIRLINES} \\
		\textbf{  } & \textbf{10\% noise} & \textbf{10\% noise} & \textbf{ } & \textbf{ }\\
		\hline
		
		\(\mathbf{H_0 : p_1 \leq p_2}\) & p = 41.3334\% & p = 99.7434\% & p = 64.4024\% & p = 99.9498\% \\
		\(\mathbf{H_1 : p_1 > p_2}\) & Z = 0.218977 & Z = -2.798629 & Z = -0.369236 & Z = -3.289466 \\
		\hline
		\(\mathbf{H_0 : p_1 \leq p_3}\) & p = 100.0000\% & p = 100.0000\% & p = 5.2861\% & p = 44.0022\% \\
		\(\mathbf{H_1 : p_1 > p_3}\) & Z = -8.448121 & Z = -17.539236 & Z = 1.617728 & Z = 0.150913 \\
		\hline
		\(\mathbf{H_0 : p_1 \leq p_4}\) & p=59.7875\% & p=7.3610\% & p=23.1987\% & \textbf{p=0.0000\%} \\
		\(\mathbf{H_1 : p_1 > p_4}\) & Z = -0.247851 & Z = 1.449419 & Z = 0.732319 & \textbf{Z = 20.871787} \\
		\hline
		
		\(\mathbf{H_0 : p_2 \leq p_1}\) & p = 58.6666\% & \textbf{p = 0.2566\%} & p = 35.5976\% & \textbf{p = 0.0502\%} \\
		\(\mathbf{H_1 : p_2 > p_1}\) & Z = -0.218977 & \textbf{Z = 2.798629} & Z = 0.369236 & \textbf{Z = 3.289466} \\
		\hline
		\(\mathbf{H_0 : p_2 \leq p_3}\) & p = 100.0000\% & p = 100.0000\% & \textbf{p = 2.3578\%} & \textbf{p = 0.0290\%} \\
		\(\mathbf{H_1 : p_2 > p_3}\) & Z = -8.666700 & Z = -14.740450 & \textbf{Z = 1.984893} & \textbf{Z = 3.440387} \\
		\hline
		\(\mathbf{H_0 : p_2 \leq p_4}\) & p=67.9686\% & \textbf{p=0.0011\%} & p=10.4849\% & \textbf{p=0.0000\%} \\
		\(\mathbf{H_1 : p_2 > p_4}\) & Z = -0.466820 & \textbf{Z = 4.247979} & Z = 1.254399 & \textbf{Z = 24.161922} \\
		\hline
		
		\(\mathbf{H_0 : p_3 \leq p_1}\) & \textbf{p = 0.0000\%} & \textbf{p = 0.0000\%} & p = 94.7139\% & p = 55.9978\% \\
		\(\mathbf{H_1 : p_3 > p_1}\) & \textbf{Z = 8.448121} & \textbf{Z = 17.539236} & Z = -1.617728 & Z = -0.150913 \\
		\hline
		\(\mathbf{H_0 : p_3 \leq p_2}\) & \textbf{p = 0.0000\%} & \textbf{p = 0.0000\%} & p = 97.6422\% & p = 99.9710\% \\
		\(\mathbf{H_1 : p_3 > p_2}\) & \textbf{Z = 8.666700} & \textbf{Z = 14.740450} & Z = -1.984893 & Z = -3.440387 \\
		\hline
		\(\mathbf{H_0 : p_3 \leq p_4}\) & \textbf{p=0.0000\%} & \textbf{p=0.0000\%} & p=94.0518\% & \textbf{p=0.0000\%} \\
		\(\mathbf{H_1 : p_3 > p_4}\) & \textbf{Z = 8.200416} & \textbf{Z = 26.851278} & Z = -1.559140 & \textbf{Z = 29.303565} \\
		\hline
		
		\(\mathbf{H_0 : p_4 \leq p_1}\) & p=40.2125\% & p=92.6390\% & p=69.7711\% & p=100.0000\% \\
		\(\mathbf{H_1 : p_4 > p_1}\) & Z = 0.247851 & Z = -1.449419 & Z = -0.517828 & Z = -20.871787 \\ \hline
		\(\mathbf{H_0 : p_4 \leq p_2}\) & p=32.0314\% & p=99.9989\% & p=81.2459\% & p=100.0000\% \\
		\(\mathbf{H_1 : p_4 > p_2}\) & Z = 0.466820 & Z = -4.247979 & Z = -0.886994 & Z = -24.161922 \\ \hline
		\(\mathbf{H_0 : p_4 \leq p_3}\) & p=100.0000\% & p=100.0000\% &  p=13.5123\% & p=100.0000\% \\
		\(\mathbf{H_1 : p_4 > p_3}\) & Z = -8.200416 & Z = -18.988307 & Z = 1.102497 & Z = -20.720896 \\ \hline

	\end{tabular}
\end{table}

The results in Table \ref{tab:HT} show that for the STAGGER dataset, the best-performing betting function is MIHNN. Specifically, when compared to IH, MIH, and CAU, the hypothesis is rejected with a p-value that approximates 0. However, there is no statistical evidence to conclude which betting function is better among IH, MIH, and CAU. In the SEA dataset, the betting function with higher accuracy is again MIHNN, because hypothesis tests against the other three betting functions reject the null hypothesis with a p-value approximating 0. We also found statistical evidence that MIH is preferable to IH, with a very low p-value of 0.26\%. Simulations on the ELEC dataset reveal statistical evidence that MIH has higher accuracy than MIHNN, because the null hypothesis is rejected with a p-value of 2.36\%. However, there is no statistical evidence to determine which betting function performs better. Finally, for the AIRLINES dataset, when we compare MIH with MIHNN, IH, and CAU, the null hypothesis is rejected with a significance level of less than 1\%, indicating that the betting function with the best performance is MIH. We have also found evidence that CAU is outperformed by all betting functions with a significance level very close to 0.

\subsection{Ensembles Performance Analysis}

In this subsection, we show the accuracies and the number of non-available predictions using the MIHNN betting function with the $r$ parameter set to 10. Figures \ref{fig:stagger_subfigures}, \ref{fig:sea_subfigures}, \ref{fig:elec_subfigures}, and \ref{fig:airlines_subfigures} illustrate the accuracy achieved with various combinations of 1 to 10 classifiers, which can help us assess the optimal number of classifiers for maximal accuracy. Moreover, Figures \ref{fig:stagger_subfiguresnans}, \ref{fig:sea_subfiguresnans}, \ref{fig:elec_subfiguresnans}, and \ref{fig:airlines_subfiguresnans} show the frequency of non-available predictions corresponding to each classifier combination. Note that other configurations show a similar pattern regarding the accuracy range and the number of available predictions.

\begin{figure}[h]
	\centering
	\includegraphics[width=0.71\textwidth]{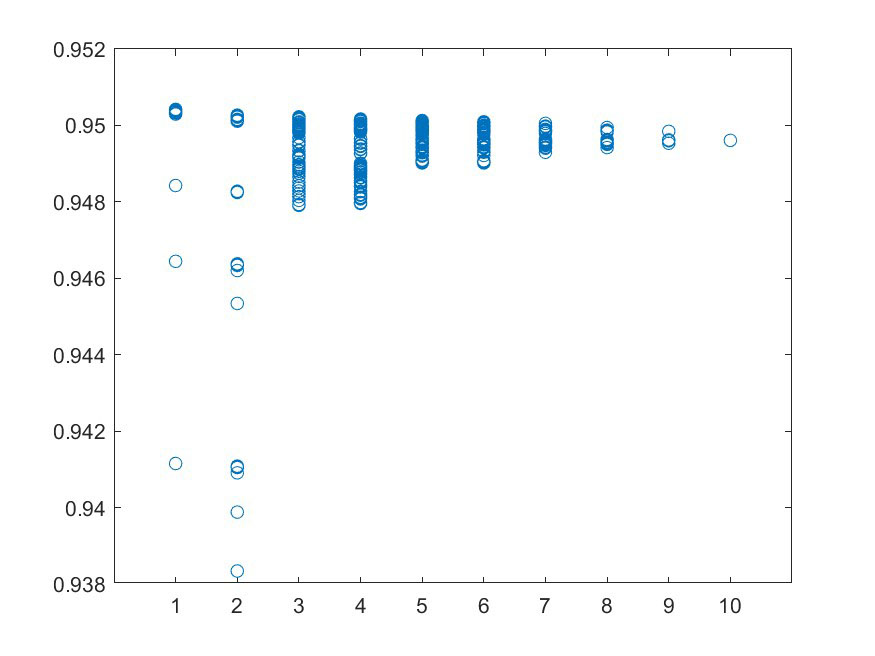}
	\caption{STAGGER dataset accuracy with R=10 and 10\% noise using the MIHNN betting function and all possible combinations of 10 classifiers.}
	\label{fig:stagger_subfigures}
\end{figure}

\begin{figure}[h]
	\centering
	\includegraphics[width=0.71\textwidth]{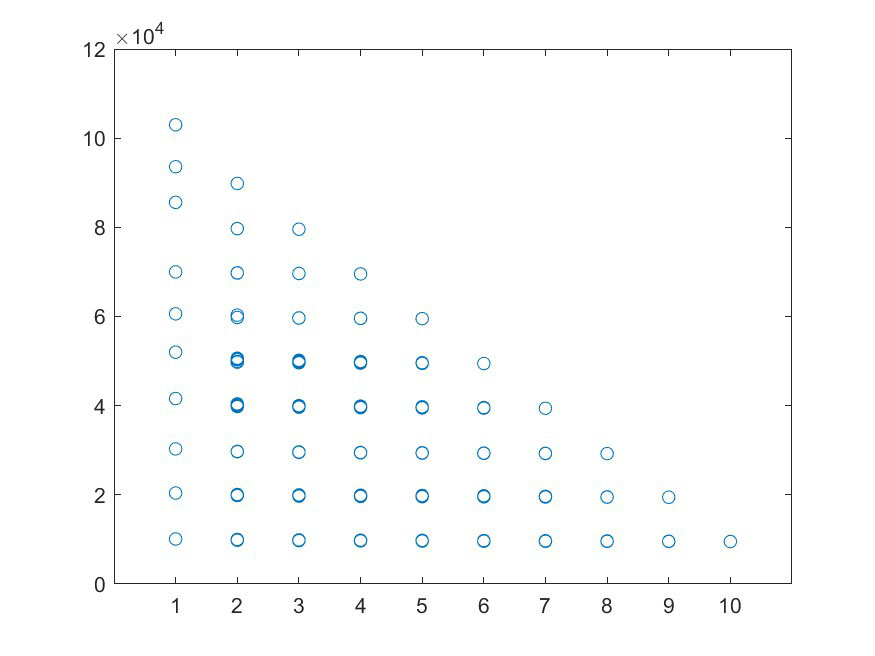}
	\caption{STAGGER dataset's number of non-available predictions with R=10 and 10\% noise using the MIHNN betting function and all possible combinations of 10 classifiers.}
	\label{fig:stagger_subfiguresnans}
\end{figure}

Figure \ref{fig:stagger_subfigures} shows that for the Stagger dataset, as the number of classifiers increases, the range of accuracies decreases. It is worth mentioning that in this dataset, there are ensemble configurations where using fewer than 10 classifiers results in higher accuracy. Additionally, Figure \ref{fig:stagger_subfiguresnans} demonstrates that many combinations of classifiers can achieve a small number of instances with no prediction. However, finding a combination of classifiers with high accuracy and simultaneously a small number of non-available predictions is a computationally expensive task. Thus, using an ensemble with 10 classifiers is a safe choice.

\begin{figure}[h]
	\centering
	\includegraphics[width=0.71\textwidth]{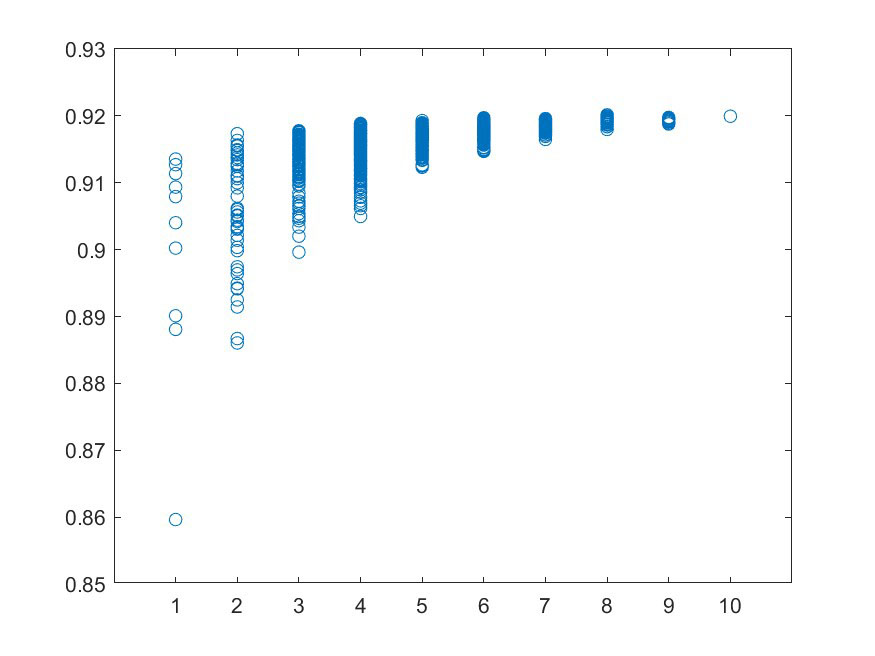}
	\caption{SEA dataset accuracy with R=10 and 10\% noise using the MIHNN betting function and all possible combinations of 10 classifiers.}
	\label{fig:sea_subfigures}
\end{figure}

\begin{figure}[h]
	\centering
	\includegraphics[width=0.71\textwidth]{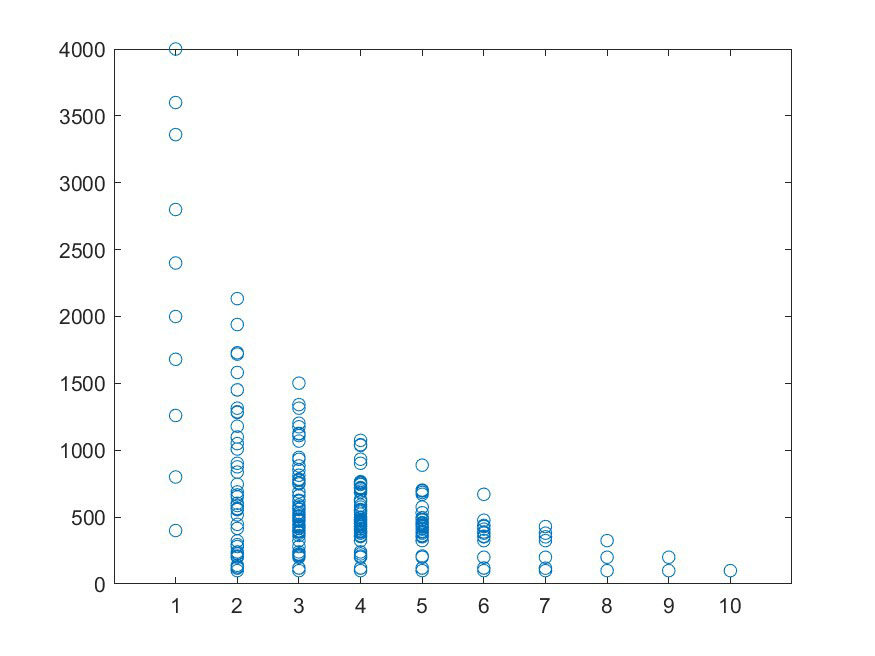}
	\caption{SEA dataset's number of non-available predictions with R=10 and 10\% noise using the MIHNN betting function and all possible combinations of 10 classifiers.}
	\label{fig:sea_subfiguresnans}
\end{figure}

Consistent with the findings above, Figure \ref{fig:sea_subfigures} shows that in the SEA dataset, as the number of classifiers increases, the range of accuracies decreases; however, we also observe an increasing trend in these accuracies. Again, in this dataset, there are configurations where using fewer than 10 classifiers results in higher accuracy. Similarly to the STAGGER dataset, Figure \ref{fig:sea_subfiguresnans} demonstrates that many combinations of classifiers can achieve a small number of instances with no predictions. Searching for the optimal combination of classifiers that provides high accuracy while simultaneously minimizing the number of non-available predictions is a computationally expensive task. Thus, once again, using an ensemble of 10 classifiers is a safe choice.

\begin{figure}[h]
	\centering
	\includegraphics[width=0.71\textwidth]{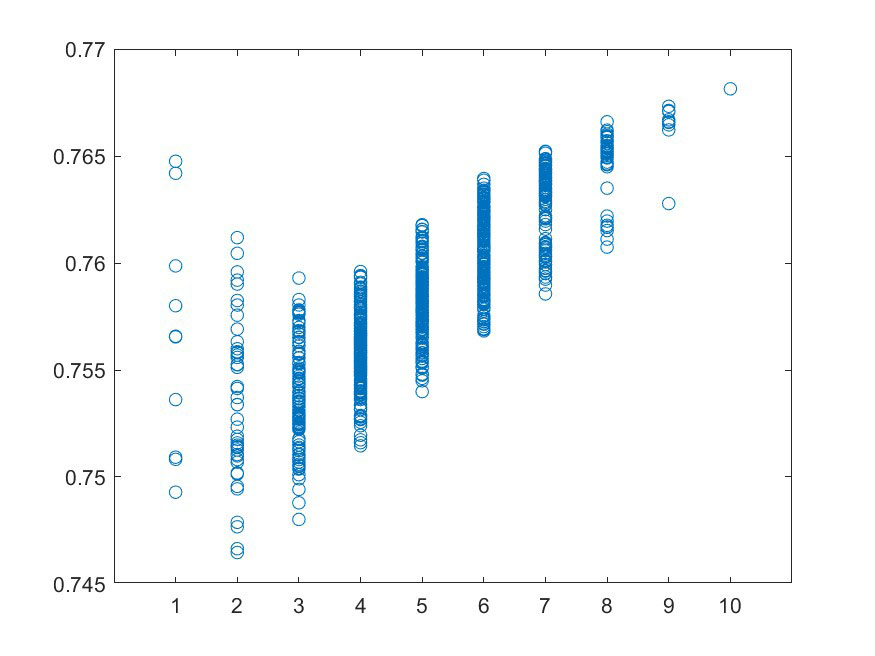}
	\caption{ELEC dataset accuracy with R=10 using the MIHNN betting function and all possible combinations of 10 classifiers.}
	\label{fig:elec_subfigures}
\end{figure}

\begin{figure}[h]
	\centering
	\includegraphics[width=0.71\textwidth]{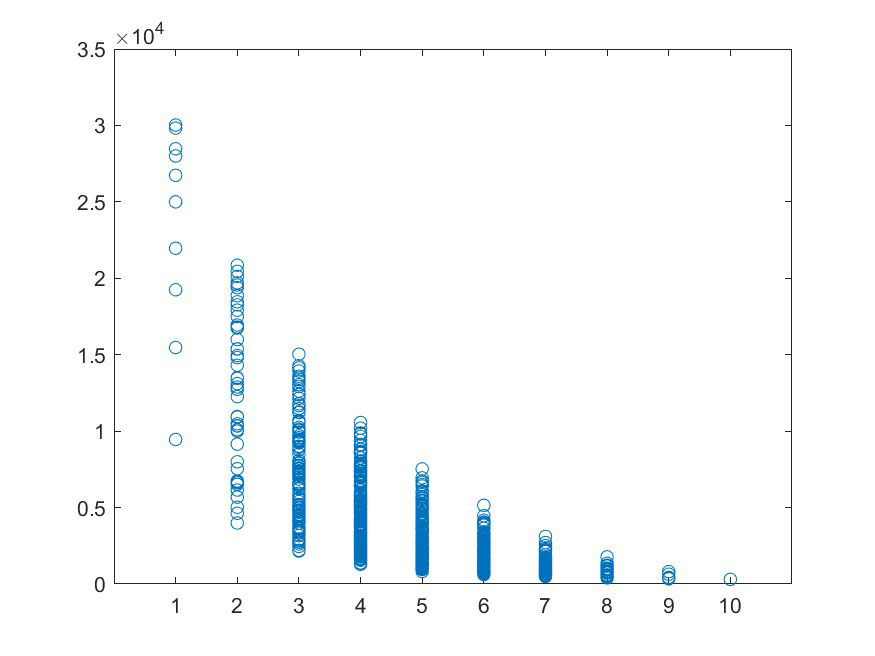}
	\caption{ELEC dataset's number of non-available predictions with R=10 using the MIHNN betting function and all possible combinations of 10 classifiers.}
	\label{fig:elec_subfiguresnans}
\end{figure}
For the ELEC dataset, using 10 classifiers consistently yields the best accuracy and the highest number of instances with available predictions. This is demonstrated in Figure \ref{fig:elec_subfiguresnans}, where ensembles with 10 classifiers provide more instances with available predictions than configurations with fewer classifiers. Additionally, as shown in Figure \ref{fig:elec_subfigures}, configurations with fewer classifiers result in fewer available predictions and lower accuracy. Therefore, using an ensemble of ten classifiers is preferable.

\begin{figure}[h]
	\centering
	\includegraphics[width=0.71\textwidth]{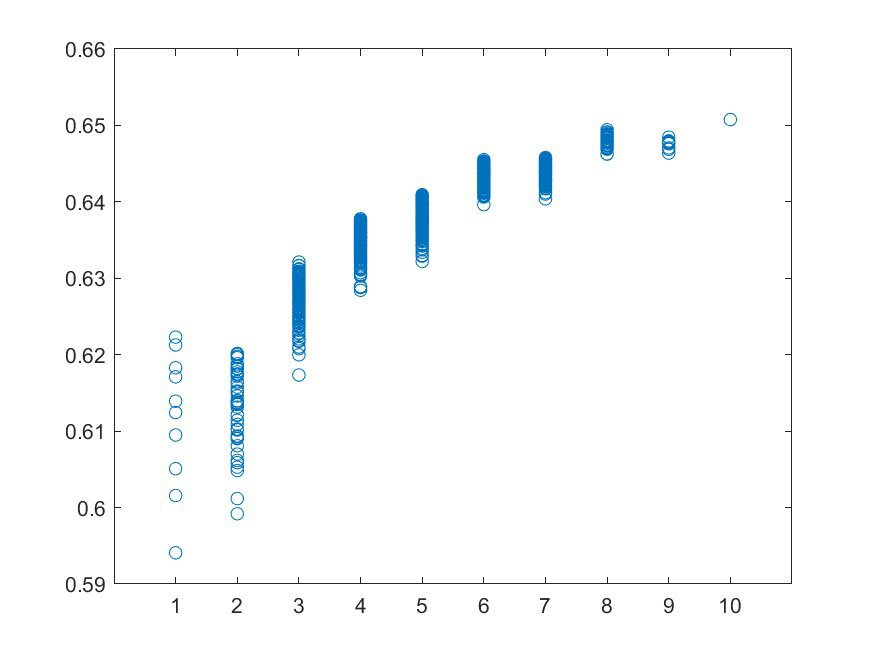}
	\caption{AIRLINES dataset accuracy with R=10 using the MIHNN betting function and all possible combinations of 10 classifiers.}
	\label{fig:airlines_subfigures}
\end{figure}

\begin{figure}[hp]
	\centering
	\includegraphics[width=0.71\textwidth]{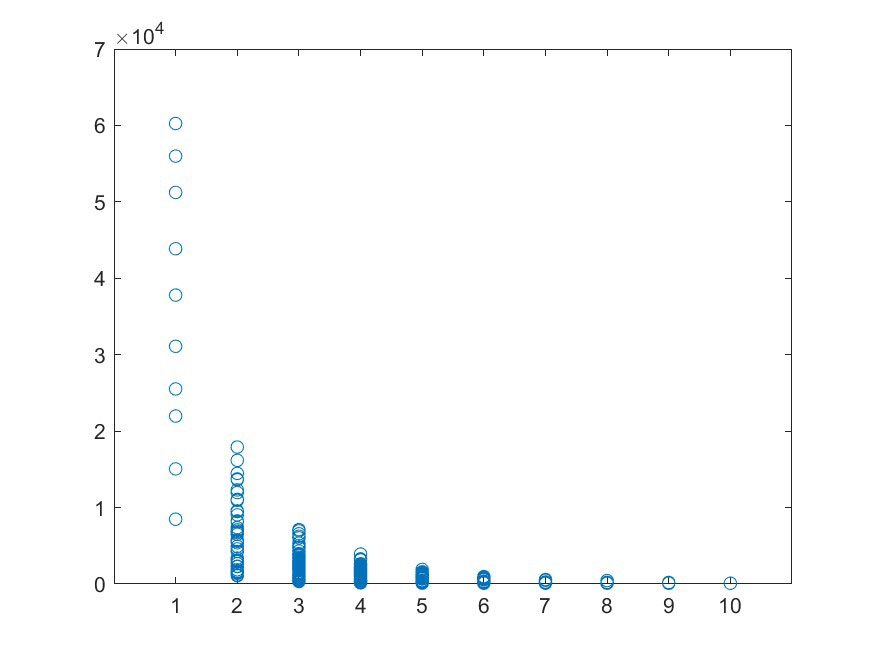}
	\caption{AIRLINES dataset's number of non-available predictions with R=10 using the MIHNN betting function and all possible combinations of 10 classifiers.}
	\label{fig:airlines_subfiguresnans}
\end{figure}

Figure \ref{fig:airlines_subfigures} indicates that for the AIRLINES dataset, the highest level of accuracy is achieved with an ensemble of 10 classifiers. Additionally, Figure \ref{fig:airlines_subfiguresnans} shows that the number of non-available predictions is the lowest when employing 10 classifiers.

The conclusion here, regarding all tested datasets, is that using an ICM ensemble consisting of 10 classifiers is a safe choice because finding the optimal combination of classifiers is a computationally expensive task. Using an ensemble of 10 ICMs allows us to maintain a low number of unavailable predictions while achieving, in most cases, performance near the optimum value.

\subsection{Discussion}

In this section, we assess the performance of our new model in comparison to the most effective models identified in our prior research, as detailed in references \cite{pmlr-v179-eliades22a} and \cite{pmlr-v204-eliades23a}, alongside several leading state-of-the-art methods. The abbreviations used herein are as follows: ICM-E* denotes ICM Ensemble* from this study, ICM-E refers to ICM Ensemble from \cite{pmlr-v204-eliades23a}, CAU stands for CAUTIOUS as per \cite{pmlr-v179-eliades22a}, AWE represents Accuracy Weighted Ensemble, DWM-NB denotes Dynamic Weighted Majority with Naive Bayes, and ARFHT is  Adaptive Random Forest with Hoeffding Tree.

Table \ref{tab:4} presents comparative results, highlighting the accuracies obtained by the optimal betting function from our recent simulations (ICM ensemble* column), three state-of-the-art algorithms referenced in \cite{Classification_comparisons}, and the most successful methodologies from our previous studies. For the STAGGER dataset, the accuracy of ICM-E* approximates the theoretical expectation for a dataset with $10\%$ noise, thereby outperforming CAU and AWE, and matching the results of ICM-E and ARFHT. In the case of the SEA dataset, also with $10\%$ noise, our new methodology surpasses the state-of-the-art and improves upon our previous results. Observing the accuracies for the ELEC dataset, it is evident that ICM-E* shows slight improvements over ICM-E, exceeds the performance of CAU, and surpasses all state-of-the-art algorithms. When comparing results for the AIRLINES dataset, ICM-E* outperforms all  methodologies, with the exception of ARFHT.

\begin{table}[ht]
	\caption{Comparison of our Icm Ensemble*  with state of the art methods and our previous studies}\label{tab:4}
	\begin{tabular}{|l|l|l|l|l|l|l|}
		\hline
		\textbf{Dataset} & \textbf{ICM-E* } & \textbf{ICM-E } & \textbf{CAU} & \textbf{AWE} & \textbf{DWE-NB} & \textbf{ARFHT} \\
		
		\hline
		STAGGER & 0.949 & 0.949 & 0.946 & 0.948 & 0.901 &0.949 \\
		SEA & 0.920 & 0.917 & 0.915 & 0.879 & 0.876 &0.841 \\
		ELEC & 0.768 & 0.767 & 0.759 & 0.756 & 0.800 & 0.857\\
		AIRLINES & 0.651 & 0.640 & 0.602 & 0.618 & 0.640 &0.666 \\
		\hline
	\end{tabular}
\end{table}

To assess whether our methodology demonstrates improvements over previous studies, especially when accuracy differences are minimal, hypothesis tests were conducted as discussed in Section \ref{HT} and presented in Table \ref{tab:HypothesisTestResults}. Here, \(p_1\), \(p_2\), \(p_3\), and \(p_4\) represent the accuracies of ICM-E*, ICM-E, CAU, AWE, DWE-NB, and ARFHT respectively. We assumed a perfect negative correlation \((-1)\) between the accuracies of the new and previous models to make hypothesis rejection more challenging. Additionally, we considered the number of instances tested,  for the STAGGER dataset, \cite{Classification_comparisons} tested on 100,000 instances. Furthermore our study involved five simulations, while \cite{Classification_comparisons} conducted only one. Rounding in reported accuracies was also considered; for instance, an accuracy reported as $0.948$ was treated as $0.9485$, while our methodologies used $15$ decimal places. The hypothesis testing framework was established with \(H_0: p_1 \leq p_2\) versus \(H_1: p_1 > p_2\), where \(p_1\) represents the accuracy of the new model and \(p_2\) represents the accuracy of previous models.

Hypothesis tests for the STAGGER dataset indicate that ICM-E* is better than CAU and DWE-NB, as it rejects the null hypothesis with a p-value close to $0$. However, there is insufficient statistical evidence to differentiate it significantly from other models, as their performance closely aligns with the theoretical expected accuracy. In the case of the SEA dataset, our methodology clearly surpasses all others, rejecting the null hypothesis with a p-value near $0$. For the ELEC dataset, there is significant statistical evidence that our methodology outperforms CAU and AWE. As for the AIRLINES dataset, our approach dominates all but the ARFHT approach.

\begin{table}[h]
	\centering
	\caption{Summary of Hypothesis Test Results Across Datasets}
	\label{tab:HypothesisTestResults}
	\begin{tabular}{|c|c|c|c|c|}
		\hline
		\textbf{} & \textbf{STAGGER} & \textbf{SEA} & \textbf{ELEC} & \textbf{AIRLINES} \\
		\hline
		\(\mathbf{H_0 : p_1 \leq p_2}\) & p = 36.5987\% & \textbf{p = 0.0000\%} & p = 41.2185\% & \textbf{p = 0.0000\%} \\
		\(\mathbf{H_1 : p_1 > p_2}\) & Z = 0.342502 & \textbf{Z = 6.172890} & Z = 0.221927 & \textbf{Z = 10.871861} \\
		\hline
		\(\mathbf{H_0 : p_1 \leq p_3}\) & \textbf{p = 0.0062\%} & \textbf{p = 0.0000\%} & \textbf{p = 0.1466\%} & \textbf{p = 0.0000\%} \\
		\(\mathbf{H_1 : p_1 > p_3}\) & \textbf{Z = 3.838780} & \textbf{Z = 11.132776} & \textbf{Z = 2.974784} & \textbf{Z = 50.446432} \\
		\hline
		\(\mathbf{H_0 : p_1 \leq p_4}\) & p = 8.2092\% & \textbf{p = 0.0000\%} & \textbf{p = 0.0031\%} & \textbf{p = 0.0000\%} \\
		\(\mathbf{H_1 : p_1 > p_4}\) & Z = 1.391136 & \textbf{Z = 90.504270} & \textbf{Z = 4.005427} & \textbf{Z = 33.896615} \\
		\hline
		\(\mathbf{H_0 : p_1 \leq p_5}\) & \textbf{p = 0.0000\%} & \textbf{p = 0.0000\%} & p = 100.0000\% & \textbf{p = 0.0000\%} \\
		\(\mathbf{H_1 : p_1 > p_5}\) & \textbf{Z = 46.230986} & \textbf{Z = 96.469046} & Z = -11.695315 & \textbf{Z = 10.872559} \\
		\hline
		\(\mathbf{H_0 : p_1 \leq p_6}\) & p = 44.5169\% & \textbf{p = 0.0000\%} & p = 100.0000\% & p = 100.0000\% \\
		\(\mathbf{H_1 : p_1 > p_6}\) & Z = 0.137876 & \textbf{Z = 161.223215} & Z = -35.295378 & Z = -16.883154 \\
		\hline
	\end{tabular}
\end{table}

\section{Conclusion}\label{sec:concl}

In this study, we introduce novel betting functions and assess their behavior in both an ICM ensemble and a single ICM setup. These new betting functions seem to enhance accuracy for both the ICM ensemble and single ICM configurations. Typically, optimal accuracy is achieved using an ICM ensemble of ten classifiers in conjunction with the Cautious Betting function, particularly when multiple density estimators are employed. The application of the Cautious Betting function with these estimators shows fast CD detection, thereby improving overall accuracy. Moreover, by setting the number of classifiers in the ensemble to ten, we not only achieve high accuracy but also maintain a low rate of unavailable predictions. This setup avoids the computationally intensive task of searching for the optimal combination of classifiers. Future research will aim to refine this model on other datasets and extend its adaptability to a more comprehensive range of streaming data scenarios, enhancing its robustness and predictive capabilities.

\backmatter

\clearpage

\section*{Declarations}

\subsection*{Funding}
Not Applicable

\subsection*{Conflicts of Interest/Competing Interests}
Not Applicable

\subsection*{Ethics Approval}
Not Applicable

\subsection*{Consent to Participate}
Not Applicable

\subsection*{Consent for Publication} 
Not Applicable

\subsection*{Availability of Data and Material}
The synthetic datasets used in this study, namely STAGGER and SEA, were generated using the MOA (Massive Online Analysis) framework. The MOA software is openly accessible at \url{https://github.com/Waikato/moa/releases}.

Regarding the real-world datasets, the Elec dataset can be found at \url{https://www.openml.org/search?type=data&sort=runs&id=151&status=active}, and the Airlines dataset is available at \url{https://www.openml.org/search?type=data&status=active&qualities.NumberOfInstances=between_100000_1000000&id=1169}.

\subsection*{Code Availability}
The custom code developed during this study is available upon request.

\subsection*{Authors' Contributions}
Both authors contributed equally to this work.

\clearpage
\bibliography{sn-bibliography}

\end{document}